\documentclass[review]{elsarticle}
\usepackage[letterpaper,top=2.5cm,bottom=2.5cm,left=2.6cm,right=2.6cm,marginparwidth=1.75cm]{geometry}
\usepackage{lineno,amsmath,amsfonts,amssymb,amsthm,mathrsfs}
\usepackage[colorlinks=true,linkcolor=blue]{hyperref}
\usepackage[linesnumbered,ruled,vlined,onelanguage]{algorithm2e}
\usepackage{booktabs}
\usepackage{tikz}
\usepackage{caption,subcaption}
\usepackage{float}

\journal{Transportation Research Part B}
\biboptions{authoryear}

\newtheorem{theorem}{Theorem}
\newtheorem{lemma}{Lemma}
\newtheorem{assumption}{Assumption}
\newtheorem{definition}{Definition}

\begin{document}

\begin{frontmatter}

\title{Macroscopic Traffic Flow Modeling with Physics Regularized Gaussian Process: A New Insight into Machine Learning Applications}

\author[1]{Yun Yuan}
\author[1]{Xianfeng Terry Yang\corref{cor1}*} 
\ead{x.yang@utah.edu} 
\author[1]{Zhao Zhang}
\author[2]{Shandian Zhe}

\address[1]{Department of Civil \& Environmental Engineering, University of Utah, Salt Lake City, UT 84112, USA}
\address[2]{School of Computing, University of Utah, Salt Lake City, UT 84112, USA}

\begin{abstract}
Despite the wide implementation of machine learning (ML) techniques in traffic flow modeling recently, those data-driven approaches often fall short of accuracy in the cases with a small or noisy dataset. To address this issue, this study presents a new modeling framework, named physics regularized machine learning (PRML), to encode classical traffic flow models (referred as physical models) into the ML architecture and to regularize the ML training process. More specifically, a stochastic physics regularized Gaussian process (PRGP) model is developed and a Bayesian inference algorithm is used to estimate the mean and kernel of the PRGP. A physical regularizer based on macroscopic traffic flow models is also developed to augment the estimation via a shadow GP and an enhanced latent force model is used to encode physical knowledge into stochastic processes. Based on the posterior regularization inference framework, an efficient stochastic optimization algorithm is also developed to maximize the evidence lowerbound of the system likelihood. To prove the effectiveness of the proposed model, this paper conducts empirical studies on a real-world dataset which is collected from a stretch of I-15 freeway, Utah. Results show the new PRGP model can outperform the previous compatible methods, such as calibrated pure physical models and pure machine learning methods, in estimation precision and input robustness.
    \end{abstract}


\end{frontmatter}

\section{Introduction}\label{sec:1}

Traffic state (i.e. flow, speed, and density) estimation (TSE) is the precursor of a variety of advanced traffic operation tasks and plays a key role in traffic management. In early stages, macroscopic traffic dynamics were found to be similar to hydrodynamics. By borrowing  concepts from the fluid mechanism, flow, speed, and density were defined and their relationship, named the fundamental diagram, was discovered. Based on these definitions, macroscopic traffic flow models were developed based on the conservation law and momentum and a set of kinematic wave models were also formulated \citep{seo2017traffic}. 
However, most models, derived under ideal theoretical conditions, require great efforts for parameter calibrations and are difficult to work with noisy and fluctuated data collected by traffic sensors.

Then to capture the measurement errors, stochastic traffic flow models were developed for the investigation and explanation of a variety of observed traffic phenomena, which are also better suited for real-time traffic state estimation and forecasting \citep{jabari2014probabilistic}.
Since the deterministic prominent models and their higher-order extensions are ill-posed, researchers developed stochastic traffic flow models in two categories. The first category used stochastic extensions \citep{gazis1971line, szeto1972application, gazis2003kalman, wang2005real, wang2007real}, which were performed by adding Gaussian noises to the model expressions and obtained real-world data were used to quantify those noises.
However, \cite{jabari2012stochastic} pointed out that those simply-noised models could lead to the possibility of: (i) causing negative sample paths and (ii) producing mean dynamics that do not coincide with the original deterministic dynamics due to nonlinearity.
The second category includes stochastic traffic models such as Botlzmann-based models \citep{prigogine1971kinetic, paveri1975boltzmann}, Markovian queuing network approaches \citep{davis1994estimating,kang1995estimation,di2010hybrid,osorio2011dynamic,jabari2012stochastic}, and cellular automaton based models \citep{nagel1992cellular, gray2001ergodic, sopasakis2006stochastic,sopasakis2012lattice}. Stochastic traffic models do not have the same concerns of the models in the first category. However, they may lose the analytical tractability \citep{jabari2013stochastic}, defined as the ability of obtaining a mathematical solution such as a closed-form expression, and are much more similar to data-driven approaches than classical analytical models. 


In view of the increasing data availability, many data-driven methods were developed because they do not require explicit theoretical assumptions and have a remarkably low computational cost in the testing phase.
In the literature, data-driven approaches include autoregressive integrated moving average \cite{zhong2004estimation}, Bayesian network \cite{ni2005markov}, kernel regression \citep{yin2012imputing}, fuzzy c-means clustering \citep{tang2015hybrid}, k-nearest neighbors clustering \citep{tak2016data}, stochastic principal component analysis \citep{li2013efficient,tan2014robust}, Tucker decomposition \citep{tan2013tensor}, deep learning \citep{duan2016efficient,polson2017deep,wu2018hybrid}, Bayesian particle filter \citep{polson2017bayesian}, etc.
However, due to the data-driven nature, those machine learning (ML) models fundamentally suffers from three scenarios: 
(i) training data are scarce and insufficient to reveal the complexity of the system, 
(ii) training data are noisy and include much incorrect/misleading information, and 
(iii) test data are far from the training examples, i.e., extrapolation. In these scenarios which are unfortunately very common in the real-world, their performance can drop dramatically along with large and/or biased estimations.  
Fig.~\ref{fig:flaw_data1} shows an example of applying a pure ML method on a dataset that contains flawed data and its biased estimation (dash line) diverges from ML methods on accurate data (solid line). Moreover, another deficiency of ML models is that they are developed as "black boxes" and researchers are hard to interpret model results. 

\begin{figure}[h!]     
    \centering
    \begin{subfigure}[b]{0.4\textwidth}
        \centering
        \includegraphics[width=\textwidth]{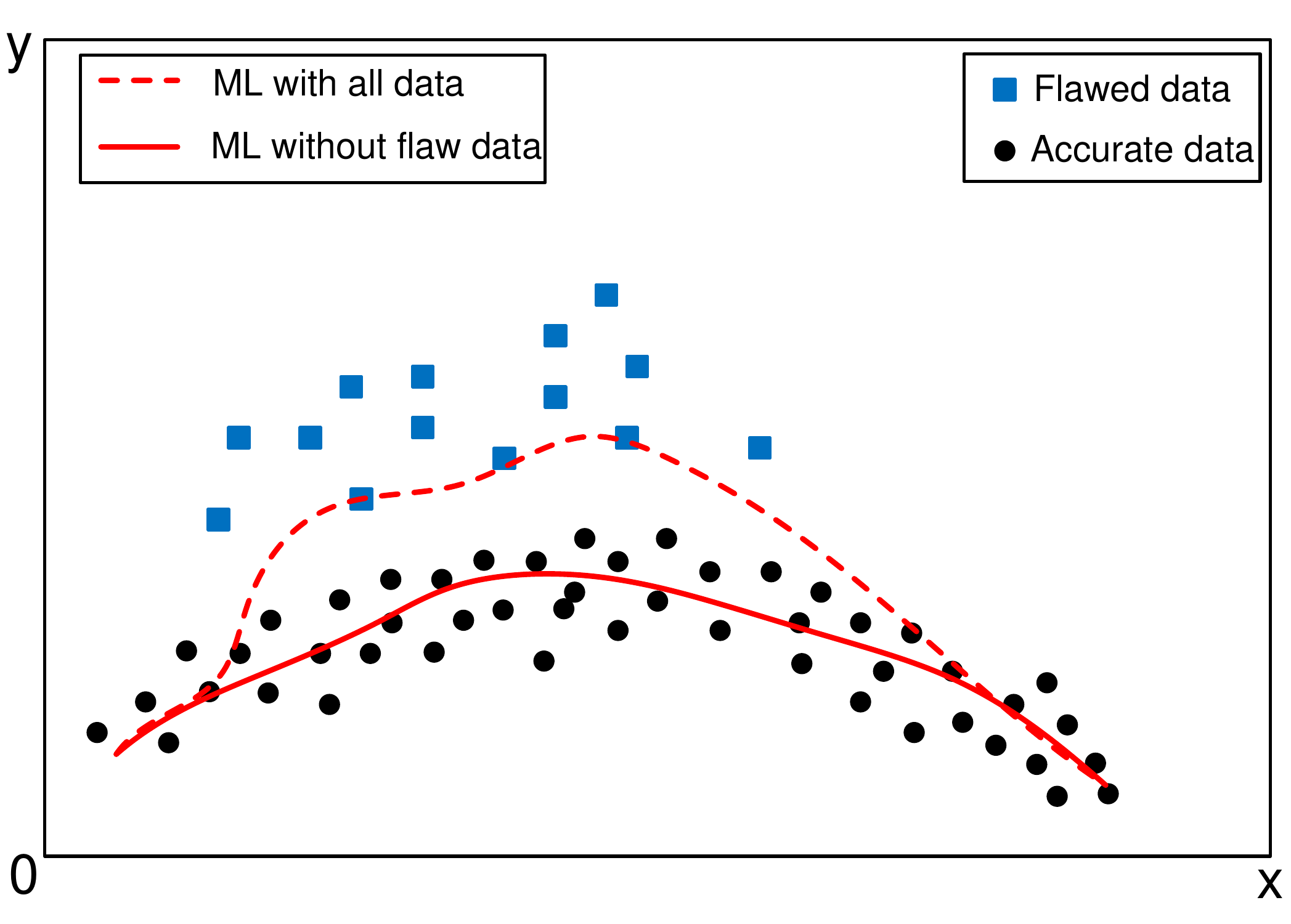}
        \caption{ML with flawed data}
        \label{fig:flaw_data1}
    \end{subfigure}
    \hfill
    \begin{subfigure}[b]{0.4\textwidth}
        \centering
        \includegraphics[width=\textwidth]{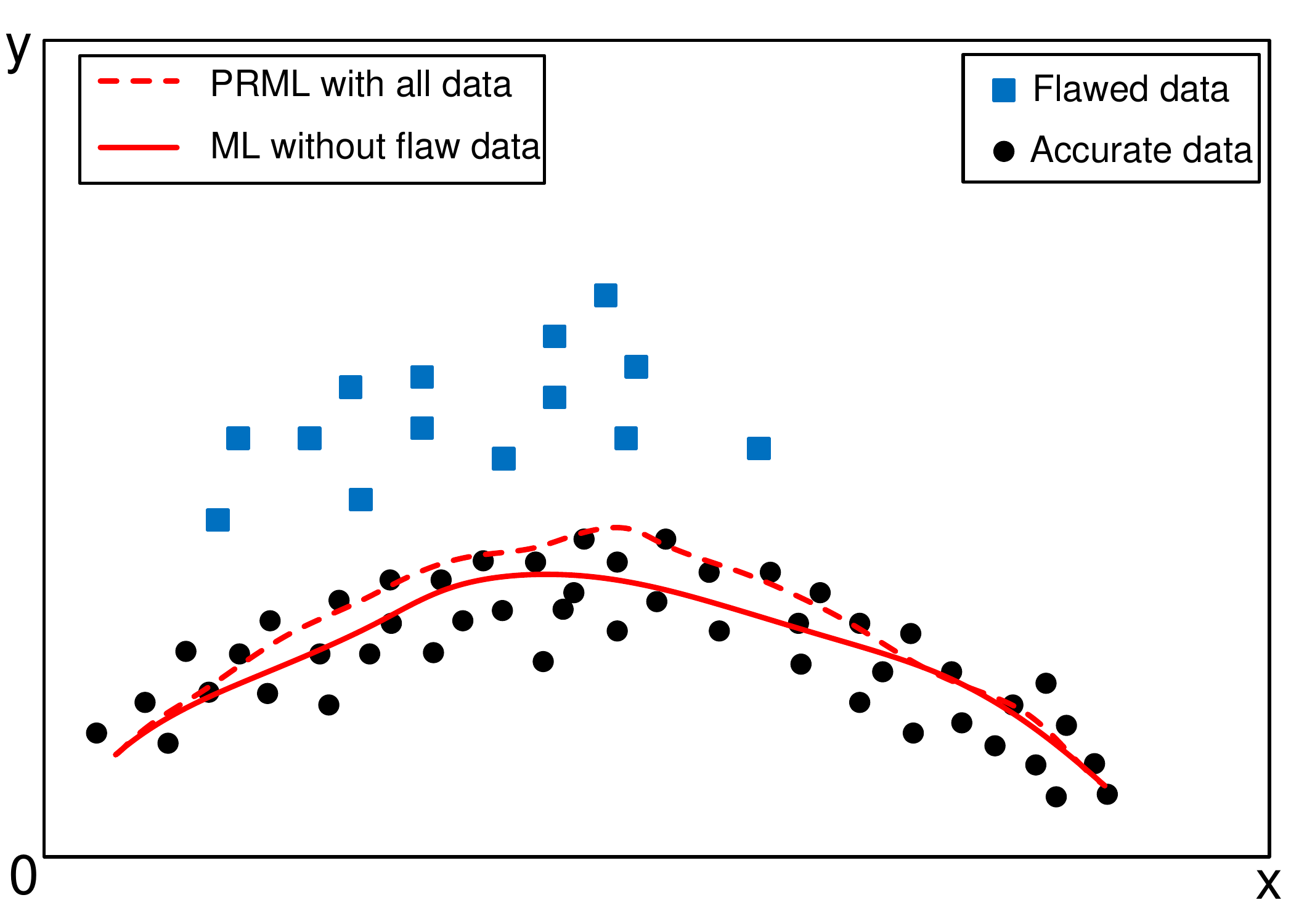}
        \caption{PRML with flawed data}
        \label{fig:flaw_data2}
    \end{subfigure}
    \vspace{-0.1in}
    \caption{Comparison between pure ML and the proposed PRML}
    \label{fig:flaw_data}
    \vspace{-0.2in}
\end{figure}

In summary, classical traffic flow models can effectively characterize the underlying mechanisms (i.e., physical processes of traffic) of transportation systems, however, are usually developed with strong assumptions, require great efforts in parameter calibrations, and fall short of capturing data uncertainties. On the other hand, the performances of pure data-driven approaches such as ML models highly depend on the data quality and their results are usually hard to be interpreted. Hence, recognizing those limitations, this research aims to develop an innovative approach, named physics regularized machine learning (PRML), to fill the gap between classical traffic flow (physical) models and ML methods. The contributions of this study are significant. Compared with physical models, the PRML can 
(1) use the ML portion to capture the uncertainties in estimation which beyond the capability of the closed-form expressions; and 
(2) eliminate the efforts in calibrating model parameter by a sequential learning process. Different from pure ML models, the PRML is 
(1) more robust under the condition of the noisy/flawed dataset as valuable knowledge from physical models can help regularize the fitting process (see Fig.~\ref{fig:flaw_data2}); and 
(2) more explainable in terms of the model performance in estimation accuracy. With this innovative modeling framework, this research is expected to bring a new insight into ML applications in transportation and build a bridge to connect the researches of classical traffic flow models and more recent data-driven approaches.

More specifically, this study develops a physics regularized Guassian process (PRGP) method for TSE by integrating three macroscopic traffic flows models with Gaussian process (GP), implementing a shadow GP to regularize the original GP, and incorporating enhanced Latent Force Models (LFM) \citep{raissi2017machine} to encode the traffic flow model knowledge.
To learn the GPs from data efficiently, this study also proposes an inference algorithm under the posterior regularization inference framework. To justify the effectiveness of the proposed methods, numerical experiments with field data are conducted on a I-15 freeway segment in Utah and the performances of PRGP models are compared with that of both classical traffic flow models and pure ML models.
To further investigate the robustness of PRGP, synthesized noises are also added to the training set and results show PRGP is much more resilient to the noisy/flawed dataset.

The remainder of this paper is organized as follows.
Section~\ref{sec:2} reviews the existing studies regarding the TSE modeling and estimation methods as well as the Gaussian process and inference methods.
In Section~\ref{sec:3}, the integrated GP and enhanced LFM for encoding physics knowledge into Bayesian statistics and the posterior regularized inference algorithm are derived.
In Section~\ref{sec:4}, the case study on a real-world data from the interstate freeway I-15 is conducted to justify the proposed methods.
The conclusion section summarizes the the critical findings and future research directions.

\section{Literature Review}\label{sec:2}
\subsection{Macroscopic Traffic flow model}
To effectively control traffic flows, TSE has been recognized as a critical fundamental task of freeway traffic management in the literature.
TSE refers to estimating a complete traffic state based on limited traffic measurement data from stationary sensors.
Key parameters, i.e. traffic flow, speed, and density, of the macroscopic traffic flow model are used to approximate the continuous traffic state with the fundamental diagram.
Deterministic traffic flow model usually consist of a conservation law equation and a fundamental relationship \citep{seo2017traffic}.
For formalization, key concepts, including cumulative flow, flow, density, speed, are defined as follows.
\begin{definition}
    The cumulative flow $N(t,x)$ is defined as the number of vehicles that passed the position $x$ by the time $t$.
\end{definition}
\begin{definition}
    The flow $q$, density $\rho$, speed $v$ are defined in Eqs.~\ref{eq:def-q}-\ref{eq:def-v}.
    \begin{equation}
        \label{eq:def-q}
        q(t,x)=\partial_t N(t,x)
    \end{equation}
    \begin{equation}
        \label{eq:def-rho}
        \rho(t,x)=-\partial_x N(t,x)
    \end{equation}
    \begin{equation}
        \label{eq:def-v}
        v(t,x)=\frac{q(t,x)}{\rho(t,x)}
    \end{equation}
\end{definition}
In traffic flow studies, researchers found the existence of the fundamental diagram (FD) to illustrate the relationship among flow, speed and density:
\begin{definition}
    The fundamental diagram is defined as the relationship among flow, speed, and density, as shown in Eqs.~\ref{eq:fd-v}-\ref{eq:fd-q}.
    \begin{equation}
        \label{eq:fd-v}
        v=V{\rho}
    \end{equation}
    \begin{equation}
        \label{eq:fd-q}
        q = \rho V(\rho)
    \end{equation}
\end{definition}
\noindent where $V(\cdot)$ denotes the density-speed function.
Macroscopic traffic flow models were proposed based on continuum fluid approximation to describe the aggregated behavior of traffic, which can generally be classified into three basic formulations.
The well-known first-order Lighthill-Whitham-Richards (LWR) model \citep{lighthill1955kinematic,richards1956shock} is formulated in Eqs.~\ref{eq:lwr1}-\ref{eq:lwr2}. 
\begin{equation}
    \label{eq:lwr1}
    \partial_t \rho + \partial_x(\rho v) = 0
\end{equation}
\begin{equation}
    \label{eq:lwr2}
    v=V(\rho)
\end{equation}
The LWR model can describe simple behaviors, such as traffic jam and shockwave, however, has limitations in reproducibility of more complex phenomena.

To overcome such limitations, second-order models use the additional momentum equation to describe the dynamics of speed.For example, Payne-Whitham (PW) model \citep{payne1971models,whitham1975linear} is formulated by Eqs.~\ref{eq:pw1}-\ref{eq:pw2}, in which Eq.\ref{eq:pw2} is the momentum equation.
\begin{equation}
    \label{eq:pw1}
    \partial_t\rho+\partial_x(\rho v) = 0
\end{equation}
\begin{equation}
    \label{eq:pw2}
    \partial_t v+v\partial_x v=-\frac{V-V(\rho)}{\tau_0}-\frac{c^2_0}{\rho}\partial_x\rho
\end{equation}
where $\tau_0$ denotes the relaxation time and $c^2_0$ denotes a parameter related to driver anticipation.
Despite the success of the PW model and its extensions \citep{papageorgiou1989macroscopic}, the PW-like models may produce non-realistic outputs, such as negative speed \citep{del1994reaction,daganzo1995requiem,papageorgiou1998some,hoogendoorn2001state}.

To overcome this limitation, another second-order Aw-Rascle-Zhang (ARZ) model \citep{aw2000resurrection,zhang2002non} is formulated in Eqs.~\ref{eq:arz1}-\ref{eq:arz2}, where another momentum equation is proposed in Eq.~\ref{eq:arz2}.
The original ARZ model was extended extensively in the literature \citep{colombo2003hyperbolic,lebacque2007generic,blandin2013phase,fan2013comparative}.
\begin{equation}
    \label{eq:arz1}
    \partial_t\rho + \partial_x(\rho v) = 0
\end{equation}
\begin{equation}
    \label{eq:arz2}
    \partial_t(v-V(\rho) + v\partial_x (v-V(\rho))=-\frac{v-V(\rho)}{\tau_0}
\end{equation}

However, it should be noted that despite of the elegance of differential equation formalization, the traffic flow model is difficult to estimate due to the nonlinearity and the measure errors of observations in the real world.
Thus, the researchers proposed advanced estimation methods to facilitate the application of the models.

\subsection{Stochastic estimation methods}
To use field data to capture traffic flow uncertainties, some estimation models with stochastic extensions are later derived\cite{seo2017traffic}.
For example, TSE is defined as Boundary Value Problem (BVP) based on partial observations (i.e. boundary conditions) \citep{coifman2002estimating,laval2012stochastic,kuwahara2015theory,blandin2013phase,fan2013comparative}. 
In solving BVPs, the boundary conditions are assumed to be correct.
However, the real-world measure error can not be ignored.

Considering system and observation noise, data assimilation or inverse modeling techniques were then developed for model estimation and calibration.
In the literature, there exist three ways to add randomness in the traffic models: 
(a) stochastic initial and boundary conditions, 
(b) stochastic source terms (e.g. inflows), and 
(c) stochastic speed-density relationship or fundamental diagram \citep{sumalee2011stochastic}.
To capture the measure error in data, a stochastic modeling method is performed by adding Gaussian noise to the traffic state estimates \cite{gazis1971line, szeto1972application, gazis2003kalman, wang2005real, wang2007real,sumalee2011stochastic}.
For example, in view of the nonlinearity of the second order traffic flow model, \cite{gazis2003kalman,wang2005real} assumed the error terms on the formula and developed extended Kalman filter (EKF) to estimate a PW-like discrete model \citep{papageorgiou1989macroscopic}.

Note that the applying EKF to non-differentiable models (e.g. Cell Transmission Model) is not rigorous \citep{blandin2012sequential}.
The unscented Kalman filter (UKF) overcomes the shortcomings of EKF by avoiding an analytical differentiation \citep{mihaylova2006unscented}.
The ensemble Kalman filter (EnKF) employs the Monte Carlo simulation to handle nonlinear and nondifferentiable systems, but is computational costly \citep{work2008ensemble}.
The particle filter (PF) uses Monte Carlo simulation and is computation-consuming as well \citep{mihaylova2004particle}.
The simulation-based methods were further extended to reduce the computational cost.

In summary, despite the adequate applications of these methods, the stochastic extension models have two critical theoretical deficiencies: (a) negative sample paths and (b) the mean dynamics that do not coincide with the original deterministic dynamics due to the nonlinearity \citep{jabari2012stochastic,jabari2013stochastic, jabari2014probabilistic, pascale2013estimation, wada2017optimization}.
In view of such deficiencies, the intractable methods, stochastic traffic flow models, were proposed in view of the tradeoff between relaxing assumptions and the model tractability, such as 
(a) Botlzmann-based methods \citep{prigogine1971kinetic, paveri1975boltzmann}, 
(b) Markovian queuing methods \citep{davis1994estimating,kang1995estimation,di2010hybrid,osorio2011dynamic,jabari2012stochastic},
(c) cellular automation based methods \citep{nagel1992cellular, gray2001ergodic, sopasakis2006stochastic,sopasakis2012lattice}.

\subsection{Data-driven method}
More recently, with much enriched data, researchers started to seek data-driven methods, such as machine learning, Bayesian statistics, etc.
Among the existing data-driven methods, Gaussian process (GP) is a powerful non-parametric function estimator and has various successful applications.
In traffic modeling, GP-based methods are applied in traffic speed imputation \citep{rodrigues2018heteroscedastic,rodrigues2018multi}, public transport flows \cite{neumann2009stacked}, traffic volume estimation and prediction \cite{xie2010gaussian}, travel time prediction \citep{ide2009travel}, driver velocity profiles \citep{armand2013modelling} and traffic congestion \cite{liu2013adaptive}. 
It can capture relationship between stochastic variables without requiring strong assumptions (such as memorylessness).

However, as a data-driven approach, GPs can perform poorly when the training data are scarce and insufficient to reflect the complexity of the system or testing inputs are far away from the training data.
Few traffic estimation methods were developed based on GPs because it's difficult to obtain deductive insights and leverage physics knowledge.

Taking advantage of valuable knowledge from physical models (i.e., classical traffic flow model), we aim to encode them into GPs to improve their performance, especially when training on scarce data and marking estimations in areas with flawed observations.
However, it shall be noted that using GP to represent physical knowledge, modeled by differential equations, has two major difficulties: 
(a) differential equations are hard to represent as a probabilistic term, such as priors and likelihoods; 
(b) in practice, physics knowledge is usually incomplete, the differential equations can include latent functions and parameters (e.g. unobserved noise, inflows, outflows), making their presentations and joint estimation with GPs even more challenging.
To better encode the differential equations in GPs, \cite{alvarez2009latent, alvarez2013linear} proposed a Latent Force Models (LFM) for training and then the estimation of GP would be based on the convolved kernel upon Green's function.
Later on, \cite{raissi2017machine} extended the framework by assuming observable noise. 
However, the assumption of LFM is too restrictive since many realistic flexible equations are nonlinear, or linear but do not have analytical Green's function.
Also, the complete kernel is still infeasible to obtain.
Thus, it is more feasible to use expressive kernels, e.g. deep kernels \citep{wilson2016deep}.

In summary, there lack a hybrid framework to consider the physics knowledge (i.e. kinematic wave differential equations and fundamental diagram) and the data-driven methods with minimal assumptions and reasonable computational cost.
This paper aims to fill the gap by proposing a Gaussian process based data-driven method considering tractable physics knowledge.
\subsection{Gaussian process and Bayesian inference}
Gaussian process is a general framework for measuring of the similarity between observations from training data to estimate the unobserved values.
\cite{rodrigues2018heteroscedastic} and \cite{rodrigues2018multi} applied the multi-output Gaussian processes to model the complex spatiotemporal patterns about incomplete traffic speed data.
The key task is to learn the kernel (i.e. covariance) function between the variables. 
The previous studies \citep{calderhead2009accelerating,barber2014gaussian,heinonen2018learning} investigated the GP ordinary differential derivatives.
They assumed the noisy forces are observable, for example, the observable noisy forces \citep{graepel2003solving}, and observable noisy forces and solutions \citep{raissi2017machine}.

To model the observable noisy forces, Latent Force Models (LFM) \citep{alvarez2009latent,alvarez2013linear} first placed a prior over the latent forces, and then derives the covariance of the solution function via the convolution operation.
Despite the successful applications, such as transcriptional regulation modeling \citep{lawrence2007modelling}, the LFM method has two critical deficiencies: (a) it requires the linear differential equations and the analytical Green's functions, which is restrictive and does not fit the traffic flow model; and (b) the convolution procedure is computationally difficult and restrictive.

To address these issues, this paper generalizes the LFM framework and enables the nonlinear differentiation to encode the physics knowledge.
To key task is to optimize model likelihood on data and a penalty term that encodes the constraints over the posterior of the latent variables.
Via the penalty term, the domain knowledge or constraints outright to the posteriors rather than through the priors and a complex, intermediate computing procedure, hence it can be more convenient and effective. 
In view of computational efficiency, this paper further employs a posterior regularization algorithms to solve the likelihood optimization problem \citep{ganchev2013cross,zhu2014bayesian,libbrecht2015entropic,song2016kernel}.
To the best of authors knowledge, this new modeling framework is innovative and has not been developed by other transportation studies yet.
The proposed method is designed to avoid the error-prone simple stochastic assumptions and leverage the physics knowledge in a data-driven framework, which also has remarkable performance in scare data situations and unobserved inflow and outflows (e.g. an arterial stretch).

\section{Methodology}\label{sec:3}
\subsection{Macroscopic traffic flow model with Physics Regularized Gaussian Process}
\subsubsection{Gaussian process}
Suppose we aim to learn a machine $\mathbf{f} :\mathbb{R}^d \rightarrow \mathbb{R}^{d^\prime}$, it will map a $d$-dimensional Euclidean space to a $d^\prime$-dimensional Euclidean space from a training set $\mathcal{D} =(\mathbf{X},\mathbf{Y})$, 
where $\mathbf{X}=[\mathbf{x}_{1},\ldots,\mathbf{x}_{N}]^\intercal$ is the input vector, 
$\mathbf{Y}=[\mathbf{y}_{1},\ldots,\mathbf{y}_{N} ]^\intercal$ is the output vector, 
$\mathbf{x}$ is the $d$ dimensional input vector, 
$\mathbf{y}$ is the $d^\prime$ dimensional output vector, 
$\mathbf{f}=[f(\mathbf{x}_{1}),\ldots,f(\mathbf{x}_{N})]^\intercal$ is the learning function, 
and $N$ refers to the sample size.
Note that $\mathbf{X},\mathbf{Y}$ may have physical meanings only in their feasible domains.
\begin{assumption}
    It is assumed that the input $\mathbf{X}$ and the true output $\mathbf{f}$ follow a multivariate Gaussian distribution as shown in Eq.~\ref{eq:pfx} , 
    where $\mathcal{N}(\cdot,\cdot)$ represents the Gaussian distribution, 
    $\mathbf{m}$ denotes the mean matrix, 
    and $\mathbf{K}$ represents the covariance matrix.
\begin{equation}
    \label{eq:pfx}
    p(\mathbf{f}|\mathbf{X})=\mathcal{N}(\mathbf{f}|\mathbf{m},\mathbf{K})
\end{equation}
\end{assumption}
Note that Gaussian process in $d$-dimension is also called Gaussian Random Field and the above definition involves the multi-dimensional outputs. 
\begin{assumption}
It is assumed that the observations $\mathbf{Y}$ have an isotropic Gaussian noise, as shown in Eq.~\ref{eq:pyf}.
    \begin{equation}
        \label{eq:pyf}
        p(\mathbf{Y}|\mathbf{f})=\mathcal{N}(\mathbf{f},\tau^{-1}\mathbf{I})
    \end{equation}
where $\tau$ refers to the inverse variance, and isotropic noise means that the noise from each dimension is independent identically distributed (i.i.d.) and of the same variance $\tau$.
\end{assumption}

Then, by Marginalizing out $\mathbf{f}$, we can obtain the marginal likelihood as shown in Eq.~\ref{eq:pyx}.
\begin{equation}
    \label{eq:pyx}
    p(\mathbf{Y}|\mathbf{X})=\mathcal{N}(\mathbf{Y}|\mathbf{0},\mathbf{K}+\tau^{-1} \mathbf{I})
\end{equation}
where the kernel matrix $\mathbf{K}$ is defined in Eq.~\ref{eq:kernel_def}.
\begin{equation}
    \label{eq:kernel_def}
    [\mathbf{K}]_{ij}=k(x_i,x_j)
\end{equation}

Commonly, Assumption~\ref{assumption:smooth}, which requires the kernel has derivatives of all orders in its domain, would also be necessary and the positive-definite kernels include linear, polynomial, radial-basis, Laplacian, etc. \citep{fasshauer2011positive}.
\begin{assumption}
    \label{assumption:smooth}
    The kernel $k(\cdot,\cdot)$ is assumed to be positive-definite and smooth. 
\end{assumption}

Given the new input $\mathbf{x}^*$, the $\mathbf{f}$ function value can be estimated based on Eq.~\ref{eq:pfxxy}.
\begin{equation}
    \label{eq:pfxxy}
    p(\mathbf{f}(\mathbf{x}^*)|\mathbf{x}^*,\mathbf{X},\mathbf{Y})=\mathcal{N}(\mathbf{f}(\mathbf{x}^*)|\mu(\mathbf{x}^*),\nu(\mathbf{x}^*))
\end{equation}
where the mean $\mu(\mathbf{x}^*)$, standard deviation $\nu(\mathbf{x}^*)$, and the kernel vector $\mathbf{k}_*$ are calculated in Eqs.~\ref{eq:mu_def}-\ref{eq:k_def}, respectively.
\begin{equation}
    \label{eq:mu_def}
    \mu(\mathbf{x}^*)=\mathbf{k}_*^\intercal (\mathbf{K}+\tau^{-1} \mathbf{I})^{-1} \mathbf{Y}
\end{equation}
\begin{equation}
    \label{eq:nu_def}
    \nu(\mathbf{x}^*)=k(\mathbf{x}^*,\mathbf{x}^*)-\mathbf{k}_*^\intercal(\mathbf{K}+\tau^{-1} \mathbf{I})^{-1} \mathbf{k}_*
\end{equation}
\begin{equation}
    \label{eq:k_def}
    \mathbf{k}_*=[k(x^*,x_1),\ldots,k(x^*,x_N)]^\intercal
\end{equation}

If the kernel $\mathbf{K}$ has been learned from data $\mathcal{D}$, the estimated output matrix $\mathbf{f(x^*)}$ can be calculated via the reparameterization \citep{kingma2014adam} as shown in Eqs.~\ref{eq:repara1}-\ref{eq:repara2}, where $\epsilon$ is standard normally distributed.
\begin{equation}
    \label{eq:repara1}
    \epsilon = \mathcal{N}(0,1)
\end{equation}
\begin{equation}
    \label{eq:repara2}
    \mathbf{f(x^*)} = \mu(\mathbf{x}_*) + \epsilon * \sqrt{\nu(\mathbf{x}^*)}
\end{equation}

Fig.~\ref{fig:gp} shows the structure of the conventional GP method, where the circled nodes denote the random vector,s the shaded node represent known vectors, and the arrows indicate the conditional probabilities.
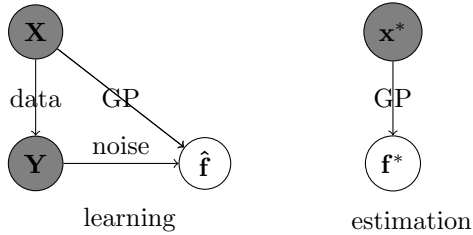
\begin{figure}[h!]
    \centering
    \begin{tikzpicture}[node0/.style={circle,draw},
                        node1/.style={circle,draw,fill=gray},
                        edge1/.style={->}]
        \node [node1] (x) at (-3.5, 3) {$\mathbf{X}$};
        \node [node1] (xs) at (1.25, 3) {$\mathbf{x^*}$};
        \node [node1] (y) at (-3.5, 1.25) {$\mathbf{Y}$};
        \node [node0] (f) at (-1.25, 1.25) {$\mathbf{\hat{f}}$};
        \node [node0] (fs) at (1.25, 1.25) {$\mathbf{f^*}$};
        \node (l) at (-2.25, 0.5) {learning};
        \node (p) at (1.5, 0.5) {estimation};
        \draw [edge1] (x) to node {data} (y);
        \draw [edge1] (y) to node[above] {noise} (f);
        \draw [edge1] (x) to node {GP} (f);
        \draw [edge1] (x) to (f);
        \draw [edge1] (xs) to node {GP} (fs);
        
    \end{tikzpicture}
    
    \caption{The conventional framework for inferring Gaussian process}
    \label{fig:gp}
       \vspace{-0.2in}
\end{figure}
\subsubsection{Latent Force Model}
In many applications, physics knowledge, expressed as differential equations, provide the insight of the system's mechanism and can be very useful for both estimation and prediction.
In the seminal work of \cite{alvarez2009latent,alvarez2013linear}, they propose latent force models (LFM) that use convolution operations to encode physics into GP kernels.
They assume the differential equations are linear and have analytical Green's function with the kernel of the latent functions.
Given this assumption, the kernel of the target function can be derived by convolving the Green's function with the kernel of the latent functions.
LFM considers $W$ output functions ${f_1(x),\ldots,f_w(x),\ldots,f_W(x)}$, and assumes each output function $f_w$ is governed by a linear differential equation.
\begin{equation}
    \label{eq:lfm}
    \mathscr{L}f_{w}(\mathbf{x}) = u_{w}(\mathbf{x})
\end{equation}
where $\mathscr{L}$ is linear differential operator \citep{courant2008methods}, and $u$ is a latent force function.

\begin{lemma}
    \label{lemma}
    If one side of Eq.~\ref{eq:lfm} is one GP, the other side is another GP.
    The covariance of a GP's derivative and the cross-covariance between the GP and its derivative can be obtained by taking derivatives over the original covariance function.
\end{lemma}
Lemma~\ref{lemma} is proven by \citep{alvarez2009latent, alvarez2013linear}.
The reasoning is based on that applying a linear differential operator on one GP results in another GP \citep{graepel2003solving} because the derivative of GP is still a GP \citep{williams2006gaussian}.

The latent force function $u$ can be further decomposed as a linear combination of several common latent force functions as follows.
\begin{equation}
    \label{eq:udecomp}
    u_w(\mathbf{x})=\sum_{r=1}^{R}s_{rw} g_{r}(\mathbf{x})
\end{equation}
where $R$ is the number of decomposed force functions, $s$ is the latent matrix.
Since $\mathscr{L}$ is linear, if we assign a GP prior over $u(x)$, $f_w(x)$ has a GP prior as well.
Moreover, if the Green's function, namely the solution of Eq.~\ref{eq:green}, is available, we can obtain Eq.~\ref{eq:green_sol}.
\begin{equation}
    \label{eq:green}
    \mathscr{L}\mathcal{G}(\mathbf{x},\mathbf{s})=\delta(\mathbf{s}-\mathbf{x})
\end{equation}
where $\delta$ is the Dirac delta function, $\mathcal{G}$ is the Green's function.
\begin{equation}
    \label{eq:green_sol}
    f_w(x)=\int \mathcal{G}(x,s)u_{i}(\mathbf{s})\textrm{d}\mathbf{s}
\end{equation}

Hence, given the kernel for $u_w$, we can derive the kernel for $f_w$ through a convolution operation which is shown in Eq.~\ref{eq:fi_kernel}.
\begin{equation}
    \label{eq:fi_kernel}
    k_{f_w}(\mathbf{x}_1,\mathbf{x}_2)=\iint \mathcal{G}(\mathbf{x}_1,\mathbf{s}_1)\mathcal{G}(\mathbf{x}_2,\mathbf{s}_2)k_{u_w}(\mathbf{s}_1,\mathbf{s}_2)\textrm{d}\mathbf{s}_1\textrm{d}\mathbf{s}_2
\end{equation}

To deal with the multiple outputs, we can place independent GP priors over common latent function $g_r$, then each $u_w$ and $f_w$ will obtain GP priors in turn.
Via a similar convolution, we can derive the kernel across different outputs (i.e. cross-covariance) $k_{f_w,f_{i'}}$.
In this way, the physics knowledge in the Green's function are hybridized with the kernel for the latent forces.
This procedure is used to learn the GP model with an convolved kernel from the training data.

\subsubsection{Augmented Latent Force Model}
Despite the elegance and success of LFM, the precondition for using LFM might be too restrictive.
To enable the kernel convolution, LFM requires that the differential equations must be linear and have analytical Green's functions.
However, many realistic differential equations from traffic flow models are either nonlinear or linear but do not possess analytical Green's functions, and therefore, cannot be exploited.
In some other cases, even with a tractable Green's function, the complete kernel of all the input variables is still infeasible to obtain.
In order to obtain an analytical kernel after the convolution, we have to convolve Green's functions with smooth kernels.
This may prevent us from integrating the physics knowledge into more complex yet highly flexible kernels, such as deep kernel \citep{wilson2016deep}.
To handle the intractable integral, we need to develop extra approximation methods, such as Monte-Carlo approximation.

Given the differential equation that describes the physics knowledge, the proposed augmented LFM equation is formulated in Eq.~\ref{eq:nonlinearlfm}.
\begin{equation}
    \label{eq:nonlinearlfm}
    \Psi f(\mathbf{x}) = g(\mathbf{x})
\end{equation}
where the differential operator $\Psi$ can be linear, nonlinear, or numerical differential operator, $g(\cdot)$ represents the unknown latent force functions, $f(\mathbf{x})$ is the function to be estimated from data $\mathcal{D}$.
We aim to create a generative component to regularize the original GP with a differential equation.
Using Augmented LFM, the differential equation is encoded to another GP, which is called a shadow GP.
To yield the numerical outputs, the kernel of the shadow GP should be efficiently learnable.
\begin{theorem}
    If one side of Eq.~\ref{eq:nonlinearlfm} is one GP, the other side is another GP.
\end{theorem}
\begin{proof}
    The reasoning is based on that applying a differential operator on one GP results in another GP.
    The regularization is fulfilled via a valid generative model component rather than the process differentiation, and hence can be applied to any linear or nonlinear differential operators.
    In view of the fact that the resultant covariance and cross-covariance are not obvious via analytical derivatives, the expressive kernels can be learned from data empirically.
\end{proof}

The original LFM starts with the RHS (right-hand side) of Eq.~\ref{eq:lfm}, assigns it a GP prior and then use the convolution operation to obtain the GP prior of the left-hand side (LHS) target function.
Since the convolution operation is an integration procedure, it can be more restrictive and challenging.
In contrast, our approach chooses a reverse direction, i.e. from LHS to RHS.
We first sample the target function with an expressive kernel, use differentiation operation to obtain the latent force, and then regularize it with another GP prior.
The differentiation operation is more flexible and convenient \citep{baydin2018automatic}, which does not need to restrict the operator and GP kernels to ensure tractable computation.
The computational challenge can be overcomed by using auto-differential libraries \citep{baydin2018automatic} and deep learning techniques (e.g. deep kernel, Tensorflow, PyTorch).
Therefore, the shadow GP can be efficiently learned from pseudo observations via differential computations.

\subsection{Physics regularized Gaussian process (PRGP)}
Involving the shadow GP, the design concept of the proposed PRGP is illustrated in Fig.~\ref{fig:framework}.
\begin{figure}[h!]
    \centering
    \begin{tikzpicture}[node0/.style={circle,draw},
                        node1/.style={circle,draw,fill=gray},
                        edge1/.style={->}]
        \node [node1] (x) at (-3.5, 3){$\mathbf{X}$};
        \node [node1] (y) at (-3.5, 1.25) {$\mathbf{Y}$};
        \node [node1] (z) at (-1.25, 3) {$\mathbf{Z}$};
        \node [node0] (f) at (-1.25, 1.25) {$\mathbf{\hat{f}}$};
        \node [node0] (g) at (0.5, 1.25) {$\mathbf{g}$};
        \node [node1] (o) at (0.5, -0.25) {$\omega$};
        \draw [edge1] (f) to node[above]{$\Psi$} (g) ;
        \draw [edge1] (g) to node {shadow GP} (o);
        \draw [edge1] (x) to node {data} (y) ;
        \draw [edge1] (y) to (f);
        \draw [edge1] (z) to node {GP} (f);
        \draw [edge1] (x) to (f);
    \end{tikzpicture}
    \caption{The proposed framework for physics regularized Gaussian process learning}
    \label{fig:framework}
       \vspace{-0.2in}
\end{figure}
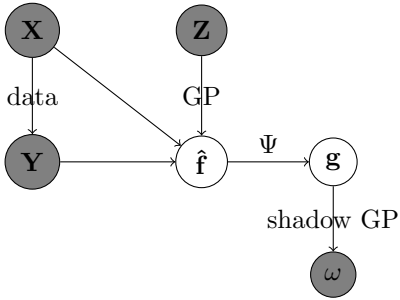
To enable Bayesian framework that incorporats the physics knowledge in Eq.~\ref{eq:nonlinearlfm}, we introduce a set of $m$ pseudo observations, $\omega=[0,\ldots,0]^\intercal$, to propose a generative component $p(\omega|\mathbf{X},\mathbf{Y})$ that acts as a physics knowledge based regularizer on the GP model $p(\mathbf{Y}|\mathbf{X})$. 
To sample the pseudo observation $\omega$, the input vector $\mathbf{Z}$ of the length $m$ is given as follow:
\begin{equation}
    \label{eq:z}
    \mathbf{Z}=[\mathbf{z}_1, \ldots, \mathbf{z}_m]^\intercal
\end{equation}
Then, we sample the posterior function values at each $\mathbf{z}_j, 1\leq j\leq m$ as shown in Eq.~\ref{eq:poszxy}.
\begin{equation}
    \label{eq:poszxy}
    p(f(\mathbf{z}_j)|\mathbf{z}_j, \mathbf{X}, \mathbf{Y}) = \mathcal{N}(f(\mathbf{z}_j)|\mu(\mathbf{z}_j),\nu(\mathbf{z}_j))
\end{equation}
We apply the differentiation operator in Eq.~\ref{eq:nonlinearlfm} to obtain the latent function values at $\mathbf{Z}$, $\mathbf{g}=[g(\mathbf{z}_1),\ldots,g(\mathbf{z}_m)]$, which is equivalent to sampling $g(\cdot)$ from the Green's function in Eq.~\ref{eq:nonlineargreen}.
\begin{equation}
    \label{eq:nonlineargreen}
    p(\mathbf{g}|\hat{\mathbf{f}})=\delta(\mathbf{g}-\Psi \hat{\mathbf{f}})
\end{equation}
Given the latent function values $\mathbf{g}$, we sample the pseudo observations $\omega$ from another GP.
\begin{equation}
    \label{eq:pogz}
    p(\omega|\mathbf{g},\mathbf{Z})=\mathcal{N}(\omega|\mathbf{g},\hat{\mathbf{K}})
\end{equation}
where $\hat{\mathbf{K}}$ is the covariance matrix and each element is calculated from the kernel $\hat{k}(\cdot,\cdot)$ in Eq.~\ref{eq:z_kernel}.
\begin{equation}
    \label{eq:z_kernel}
    [\hat{\mathbf{K}}]_{ij} = \hat{k}(\mathbf{z}_i,\mathbf{z}_j)
\end{equation}
Considering the symmetry property of the Gaussian distribution shown in Eq.~\ref{eq:symmetry}, the sampling of the pseudo observations in essence is equivalent to placing another GP prior over the sampled latent force function $\mathbf{g}$.
Therefore, this GP prior regularizes the sampled latent function.
Through the differential operator $\Psi$, the regularization propagates back to the target machine $\mathbf{f}(\cdot)$.
\begin{equation}
    \label{eq:symmetry}
    p(\omega|\mathbf{g},\hat{\mathbf{K}})=p(\mathbf{g}|\omega,\hat{\mathbf{K}})=p(\Psi \mathbf{f}|\mathbf{\omega},\hat{\mathbf{K}})
\end{equation}

Thus, the joint probability of the generative component is broken into four parts, as shown in Eq.~\ref{eq:ogfzxy}.  
\begin{equation}
    \label{eq:ogfzxy}
    p(\omega,\mathbf{g},\hat{\mathbf{f}},\mathbf{Z}|\mathbf{X},\mathbf{Y})=p(\mathbf{Z})p(\hat{\mathbf{f}}|\mathbf{Z},\mathbf{X},\mathbf{Y})p(\mathbf{g}|\hat{\mathbf{f}})p(\omega|\mathbf{g})
\end{equation}
where the prior of the $m$ input locations, $p(\mathbf{Z})$, $p(\hat{\mathbf{f}},\mathbf{Z},\mathbf{X},\mathbf{Y})$, and $p(\omega|\mathbf{g})$ are given by Eqs.~\ref{eq:priorz}-\ref{eq:pog}, respectively. Nals note that when no extra knowledge is available, $\mathbf{z}_{j}$ can be uniformly distributed assumedly. 

\begin{equation}
    \label{eq:priorz}
    p(\mathbf{Z})=\Pi_{j=1}^m p(\mathbf{z}_{j})
\end{equation}
\begin{equation}
    \label{eq:fzxy}
    p(\hat{\mathbf{f}},\mathbf{Z},\mathbf{X},\mathbf{Y})=\Pi_{j=1}^m[\mathcal{N}(\hat{\mathbf{f}}(\mathbf{z}_{j})|\mu(\mathbf{z}_{j}),\nu(\mathbf{z}_{j}))]
\end{equation}
\begin{equation}
    \label{eq:pog}
    p(\omega|\mathbf{g})=\mathcal{N}(\omega|\mathbf{g},\hat{\mathbf{K}})
\end{equation}
\subsection{Posterior regularized inference algorithm}
Posterior regularization is a powerful inference methodology in the Bayesian stochastic modeling framework \citep{ganchev2010posterior}.
The objective includes the model likelihood on data and a penalty term that encodes the constrains over the posterior of the latent variables.
Via the penalty term, we can incorporate our domain knowledge or constrains outright to the posteriors, rather than through the priors and a complex, intermediate computing procedure.
A variety of successful posterior regularization algorithms have been proposed \citep{he2013graph,ganchev2013cross,zhu2014bayesian,libbrecht2015entropic,song2016kernel}.
Hence it can be more convenient and effective.
For efficient model inference, we marginalize out all latent variables in the joint probability to avoid estimating extra approximate posteriors.
Then we derive a convenient evidence lower bound to enable the reparameterization.
Using the reparameterization and auto-differentiation libraries, we develop an efficient stochastic optimization algorithm based on the posterior regularization inference framework \citep{ganchev2010posterior}.

The proposed inference algorithm is derived as follows.
The generative component in Eq.~\ref{eq:ogfzxy} is bind to the original GP in Eq.~\ref{eq:pyx} to obtain a new principled Bayesian model.
The joint probability is given by Eq.~\ref{eq:yogfzx}.
\begin{equation}
    \label{eq:yogfzx}
    p(\mathbf{Y},\omega,\mathbf{g}, \hat{\mathbf{f}},\mathbf{Z}|\mathbf{X})=p(\mathbf{Y}|\mathbf{X})p(\omega,\mathbf{g},\hat{\mathbf{f}},\mathbf{Z}|\mathbf{X},\mathbf{Y})
\end{equation}
We first marginalize out all the latent variables in the generative component to avoid approximating their posterior in Eq.~\ref{eq:oyx}.
\begin{equation}
    \label{eq:oyx}
    \begin{split}
        p(\omega|\mathbf{X},\mathbf{Y})&=\iiint [p(\omega,\mathbf{g},\hat{\mathbf{f}},\mathbf{Z}|\mathbf{X},\mathbf{Y})\textrm{d}\mathbf{Z}\textrm{d}\mathbf{g}\text{d}\hat{\mathbf{f}}]\\
        &=\iint [p(\mathbf{Z})p(\hat{\mathbf{f}}|\mathbf{Z},\mathbf{X},\mathbf{Y})p(\omega|\Psi\hat{\mathbf{f}},\hat{\mathbf{K}})\textrm{d}\mathbf{Z}\textrm{d}\hat{\mathbf{f}}]\\
        &=\iint [p(\mathbf{Z})p(\hat{\mathbf{f}}|\mathbf{Z},\mathbf{X},\mathbf{Y})\mathcal{N}(\omega|\Psi \hat{\mathbf{f}},\hat{\mathbf{K}})\textrm{d}\mathbf{Z}\textrm{d}\hat{\mathbf{f}}]\\
        &=\mathbb{E}_{p(\mathbf{Z})}\mathbb{E}_{p(\hat{\mathbf{f}}|\mathbf{Z},\mathbf{X},\mathbf{Y})}\mathcal{N}(\Psi \mathbf{\hat{f}}|\mathbf{0},\hat{\mathbf{K}})
    \end{split}
\end{equation}

The parameter $\gamma \geq 0$ is used to control the strength of regularization effect.
\begin{equation}
    \label{eq:pyox}
    p(\mathbf{Y},\omega|\mathbf{X})=p(\mathbf{Y}|\mathbf{X})p(\omega|\mathbf{X},\mathbf{Y})^\gamma
\end{equation}

The objective is to maximize the log-likelihood in Eq.~\ref{eq:lllh}.
\begin{equation}
    \label{eq:lllh}
    \begin{split}
        \log [p(\mathbf{Y},\mathbf{\omega}|\mathbf{X})]=&\log [p(\mathbf{Y}|\mathbf{X})] + \gamma\log [p(\omega|\mathbf{X},\mathbf{Y})]\\
        =&\log [(\mathcal{N}(\mathbf{Y}|\mathbf{0},\mathbf{\hat{K}}+\tau^{-1} \mathbf{I}))]\\
        &+\gamma \log [\mathbb{E}_{p(\mathbf{Z})} \mathbb{E}_{p(\hat{\mathbf{f}}|\mathbf{Z},\mathbf{X},\mathbf{Y})} [\mathcal{N}(\Psi \hat{\mathbf{f}}|\mathbf{0},\hat{\mathbf{K}})]]
    \end{split}
\end{equation}

However, the log-likelihood is intractable due to the expectation inside the logarithm term.
To address this problem, the Jensen's inequality is used to obtain an evidence lower bound $\mathcal{L}$ in Eq.~\ref{eq:lllhlb}.
\begin{equation}
    \label{eq:lllhlb}
    \begin{split}
        \log [p(\mathbf{Y},\mathbf{\omega}|\mathbf{X})]\geq \mathcal{L}= & \log[\mathcal{N}(\mathbf{Y}|\mathbf{\omega},\hat{\mathbf{K}}+\tau^{-1}\mathbf{I})]\\
        & +\gamma \mathbb{E}_{p(\mathbf{z})} \mathbb{E}_{p(\hat{\mathbf{f}}|\mathbf{Z},\mathbf{X},\mathbf{Y})} [\log [\mathcal{N}(\Psi \hat{\mathbf{f}}|\omega,\hat{\mathbf{K}})]]
    \end{split}
\end{equation}

The existence of the general evidence lowerbound (ELBO) of a posterior distribution is proved with analyzing a decomposition of the Kullback-Leibler (KL) divergence by \cite{bishop2006pattern}.
Thus, we can obtain the ELBO of the log-likelihood in Eq.~\ref{eq:lllhlb}. 
However, the ELBO is still intractable due to the non-analytical expectation term.
In view of the expectation is out of the logarithm, we can maximize $\mathcal{L}$ via stochastic optimization shown in Alg.~\ref{alg:1}.

   \vspace{0.2in}
\begin{algorithm}[H]
    \caption{The stochastic inference algorithm}
    \label{alg:1}
    \SetAlgoLined
    \KwResult{Learned kernel parameters}
     Initialization\;
     \While{not reach stopping criteria}{
        Sample a set of input locations $\mathbf{Z}$\;
        Estimate the mean $\mu$ and the variance $\nu$ of $\mathbf{f}$ in Eqs.~\ref{eq:mu_def}-\ref{eq:nu_def}\;
        Generate a parameterized sample of the posterior target function values $\hat{\mathbf{f}}$ by the reparameterization in Eqs.~\ref{eq:repara1}-\ref{eq:repara2}\;
        Substitute the parameterized samples $\hat{\mathbf{f}}$ to obtain the unbiased estimated ELBO $\tilde{\mathcal{L}}$ in Eq.~\ref{eq:lllhlb}\;
        Calculate $\nabla_\theta\tilde{\mathcal{L}}$, an unbiased stochastic gradient of $\tilde{\mathcal{L}}$ via the auto-differential technique\;
        Update the parameters $\theta$ via the gradient decent shown in Eq.~\ref{eq:gradient} \;
     }
\end{algorithm}
\begin{equation}
    \label{eq:gradient}
    \theta^{t+1}=\theta^{t}+\alpha\nabla_\theta\tilde{\mathcal{L}}
\end{equation}
where $\alpha$ refers to the learning rate and $\theta$ denotes all trainable parameters.

To prove the correctness of Alg.~\ref{alg:1}, we need to prove the correctness of employing a regularization via ELBO as follows.
\begin{theorem}
    Maximizing the lowerbound of the log-likelihood is equivalent to a soft constraint over the posterior of the target function in the original GP.
\end{theorem}
\begin{proof}
    While the proposed inference algorithm is developed for a hybrid model rather than pure GP \citep{ganchev2010posterior}, the evidence lower bound optimized by Alg.~\ref{alg:1} is a typical posterior regularization objective that estimates a pure GP model and meanwhile penalizes the posterior of the target function to encourage a consistency with the differential equations.
    Jointly maximizing the term \[\mathbb{E}_{p(\mathbf{z})} \mathbb{E}_{p(\hat{\mathbf{f}}|\mathbf{Z},\mathbf{X},\mathbf{Y})} [\log [\mathcal{N}(\Psi \hat{\mathbf{f}}|\omega,\hat{\mathbf{K}})]]\] in the lowerbound of the log-likelihood $\mathcal{L}$ encourages all the possible latent force functions that are obtained from the target function $f(\cdot)$ via the differential operator $\Psi$ should be considered as being sampled from the same shadow GP.
    This can be viewed as a soft constraint over the posterior of the target function in the original GP model.
    Therefore, while being developed for inference of a hybrid model, the algorithm is equivalent to estimating the original GP model with some soft constraints on its posterior distribution.
    Thus, the physics knowledge regularizes the learning of the target function in the original GP.
\end{proof}
To apply the proposed method with multiple differential equations (i.e. FD, conservation law, momentum), 
Fig.~\ref{fig:traffic} shows the multi-equation multi-output framework of applying the proposed method to model the stochastic traffic flow process.
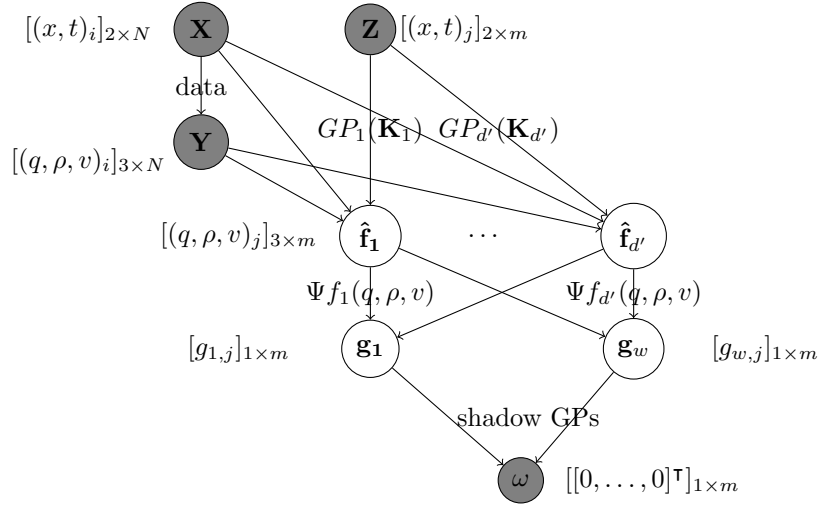
\begin{figure}[h!]
    \centering
    \begin{tikzpicture}[node0/.style={circle,draw},
                        node1/.style={circle,draw,fill=gray},
                        edge1/.style={->}]
        \node [node1] (x) at (-3.5, 3){$\mathbf{X}$};
        \node [node1] (y) at (-3.5, 1.5) {$\mathbf{Y}$};
        \node [node1] (z) at (-1.25, 3) {$\mathbf{Z}$};
        \node [node0] (f1) at (-1.25, 0.25) {$\mathbf{\hat{f}_1}$};
        \node [node0] (g1) at (-1.25, -1.25) {$\mathbf{g_1}$};
        \node (d) at (0.25, 0.25) {$\ldots$};
        \node [node0] (f2) at (2.25, 0.25) {$\mathbf{\hat{f}}_{d^\prime}$};
        \node [node0] (g2) at (2.25, -1.25) {$\mathbf{g}_w$};
        \node [node1] (o) at (0.75, -3) {$\omega$};

        \node (input) at (-5, 3) {$[(x,t)_i]_{2\times N}$};
        \node (output) at (-5, 1.25) {$[(q,\rho,v)_i]_{3\times N}$};
        \node (p_intput) at (0, 3) {$[(x,t)_j]_{2\times m}$};
        \node (p_output) at (-3, 0.25) {$[(q,\rho,v)_j]_{3\times m}$};
        \node (g_output1) at (-3, -1.25) {$[g_{1,j}]_{1\times m}$};
        \node (g_output2) at (4, -1.25) {$[g_{w,j}]_{1\times m}$};
        \node (0) at (2.5, -3) {$[[0,\ldots,0]^\intercal]_{1\times m}$};
        \draw [edge1] (f1) to node {$\Psi f_1(q,\rho,v)$} (g1) ;
        \draw [edge1] (f2) to node {$\Psi f_{d^\prime}(q,\rho,v)$} (g2) ;
        \draw [edge1] (f1) to (g2);
        \draw [edge1] (f2) to (g1);              
        \draw [edge1] (g1) to node[anchor=west] {shadow GPs} (o);
        \draw [edge1] (g2) to (o);
        \draw [edge1] (x) to node {data} (y) ;
        \draw [edge1] (y) to (f1);
        \draw [edge1] (y) to (f2);
        \draw [edge1] (z) to node {$GP_1 (\mathbf{K}_1)$} (f1);
        \draw [edge1] (z) to node {$GP_{d^\prime} (\mathbf{K}_{d^\prime})$} (f2);
        \draw [edge1] (x) to (f1);
        \draw [edge1] (x) to (f2);
    \end{tikzpicture}
    \caption{The proposed framework for multi-output multi-equation PRGP learning}
    \label{fig:traffic}
       \vspace{-0.2in}
\end{figure}

The log-likelihood and the ELBO of the traffic flow model can be formulated in Eq.~\ref{eq:lllhlb_traffic}.
\begin{equation}
    \label{eq:lllhlb_traffic}
    \begin{split}
        \log [p(\mathbf{Y},\mathbf{\omega}|\mathbf{X})]\geq \mathcal{L}=& \sum_{i=1}^{d^\prime} \log[\mathcal{N}([\mathbf{Y}]_i|\mathbf{\omega},\hat{\mathbf{K}}_i+\tau^{-1}\mathbf{I})]\\
        & +\sum_{w=1}^W\gamma_w \mathbb{E}_{p(\mathbf{z})} \mathbb{E}_{p(\hat{\mathbf{f}}_w|\mathbf{Z},\mathbf{X},\mathbf{Y})} [\log [\mathcal{N}(\Psi \hat{\mathbf{f}}_w|\omega,\hat{\mathbf{K}}_w)]]
    \end{split}
\end{equation}
\subsubsection{Expressive kernels}
The expressive kernels are defined as the non-parametric smooth covariance functions, such as the well-known Squared Exponential Automatic Relevance Determination (SEARD) Kernel, and Radial Basis Function (RBF) kernel \citep{bishop2006pattern}, and deep kernels \citep{wilson2016deep}. The employed kernel functions are shown as follows:

The SE-ARD kernel is formulated in Eq.~\ref{eq:seard}.
\begin{equation}
    \label{eq:seard}
    k(\mathbf{x}_i,\mathbf{x}_j)=\sigma^2\exp(-(\mathbf{x}_i-\mathbf{x}_j)^\intercal\textrm{diag}(\mathbf{\eta}(\mathbf{x}_i-\mathbf{x}_j)))
\end{equation}
where $\textrm{diag}(\cdot)$ represents the diagonal matrix, $\sigma$ and $\eta$ are kernel parameters.

The RBF kernel is formulated in Eq.~\ref{eq:rbf}.
\begin{equation}
    \label{eq:rbf}
    k(\mathbf{x}_i,\mathbf{x}_j)=\exp(-\frac{(||\mathbf{x}_i-\mathbf{x}_j||)^2}{2\sigma^2})
\end{equation}
where $\sigma$ is the kernel parameter.

\subsubsection{Algorithm complexity}
The time complexity of the inference of the original GP is $O(N^3)$.
The time complexity of the inference of the shadow GP is $O(m^3)$.
Thus, the total time complexity for the inference of two GPs is $O((Nd^\prime)^3+m^3)$.
To store the kernel metrics of original GP and the shadow GP, the space complexity is $O((Nd^\prime)^2+m^2)$.
In the testing phase, the time complexity of the model estimation is marginal (less than $1$ ms) empirically.

\subsection{Physics regularized traffic state estimation}
To apply the proposed method, the traffic flow models need to be converted to the form of Eq.~\ref{eq:nonlinearlfm}. In this study, we aims to encode three classical traffic flow models, LWR, PW, and ARZ, into the GP and compare their performance under the framework of PRGP. More specifically, the converted LWR, PW, ARZ models are presented as follows.
In PRGP, the stochastic conservation law of LWR is formulated in Eq.~\ref{eq:lwr_gp}.
\begin{equation}
    \label{eq:lwr_gp}
    \Psi f_1(q,\rho,v) = \partial_t\rho+\partial_x q = g_1
\end{equation}
The stochastic PW model is formulated in Eqs.~\ref{eq:pw1_gp}-\ref{eq:pw2_gp}.
\begin{equation}
    \label{eq:pw1_gp}
    \Psi f_1(q, \rho,v) = \partial_t\rho+\partial_x(\rho v) = g_1
\end{equation}
\begin{equation}
    \label{eq:pw2_gp}
    \Psi f_2(q, \rho,v) = \partial_t v+v\partial_x v+\frac{V-V(\rho)}{\tau_0}+\frac{c^2_0}{\rho}\partial_x\rho = g_2
\end{equation}
And the stochastic ARZ model is formulated in Eqs.~\ref{eq:arz1_gp}-\ref{eq:arz2_gp}.
\begin{equation}
    \label{eq:arz1_gp}
    \Psi f_1(q, \rho,v) = \partial_t\rho+\partial_x(\rho v) = g_1
\end{equation}
\begin{equation}
    \label{eq:arz2_gp}
    \Psi f_2(q, \rho,v) = \partial_t(v-V(\rho) + v\partial_x (v-V(\rho))+\frac{v-V(\rho)}{\tau_0} = g_2
\end{equation}

\section{Numerical Tests with Field Data}\label{sec:4}
\subsection{Case setting}
To evaluate the performance of the proposed PRML framework, We applied the three PRGP models to estimate the traffic flow in a stretch of the interstate freeway I-15 across Utah, U.S.
The Utah Department of Transportation (UDOT) has installed sensors every a few miles along the freeway.
Each sensor counts the number of vehicles passed every minute, measures the speed of each vehicle, and sends the data back to a central database, named Performance Measurement System (PeMS).
The collected real-time data and road conditions are available online and can be accessed by the public.
For model evaluations, the data, from August 5, 2019 to August 11, 2019, were collected by four sensors on the I-15, Utah.  
The input variables include the location coordinates of each sensor and the time of each read.
The studied stretch is illustrated in Fig.~\ref{fig:stretch}, where the yellow line indicates the studied freeway segments and the blue bars represent the locations of traffic detectors.
In the case, the data is shuffled and randomly split into the training set and testing set separately. 
\begin{figure}[h!]
    \centering
    \includegraphics[height=8cm]{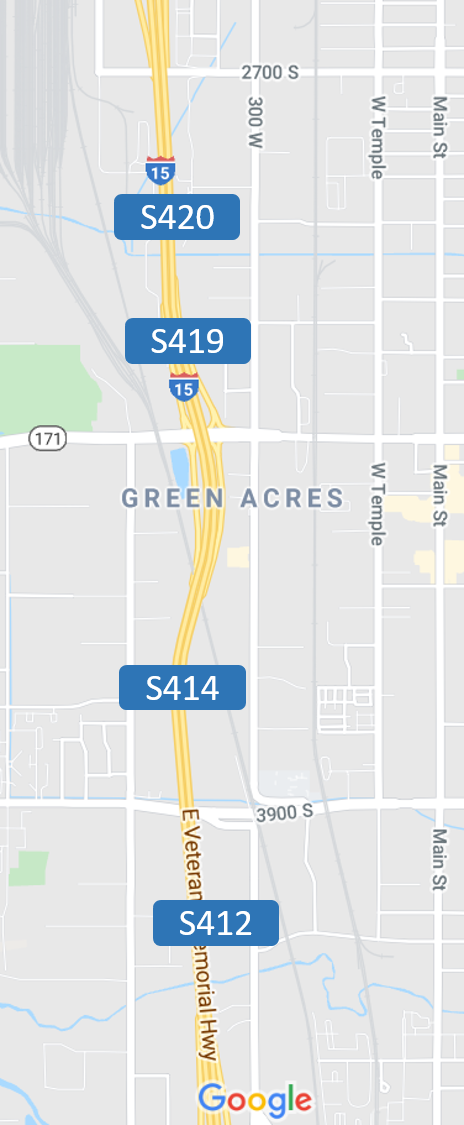}
    \caption{The stretch of the studied freeway segment which includes four detectors}
    \label{fig:stretch}
\end{figure}
\subsection{Implementation}
The deep kernel can be in any neural network structure, such as the feed-forward neural network, and can be fine-tuned to achieve better empirical results.
Incorporating the SEARD and RBF kernels, the compound kernel of the $d'$-dimensional original GP and the $W$-dimensional shadow GP are computed in Fig.~\ref{fig:implement}.
The procedure for estimating the target traffic state $q,v$ of any given input $x,t$ is illustrated in Fig.~\ref{fig:predict}.
In the multi-output multi-equation PRGP, the $d'$-dimension means for each dimension of $\mathbf{y}$ creating one compound kernel, and the $W$-dimension means for each differential equation creating one compound kernel.
Note that the structure of the GPs can be fine-tuned to achieve better empirical performance.
\begin{figure}[h!]
    \centering
    \includegraphics[width=0.7\textwidth]{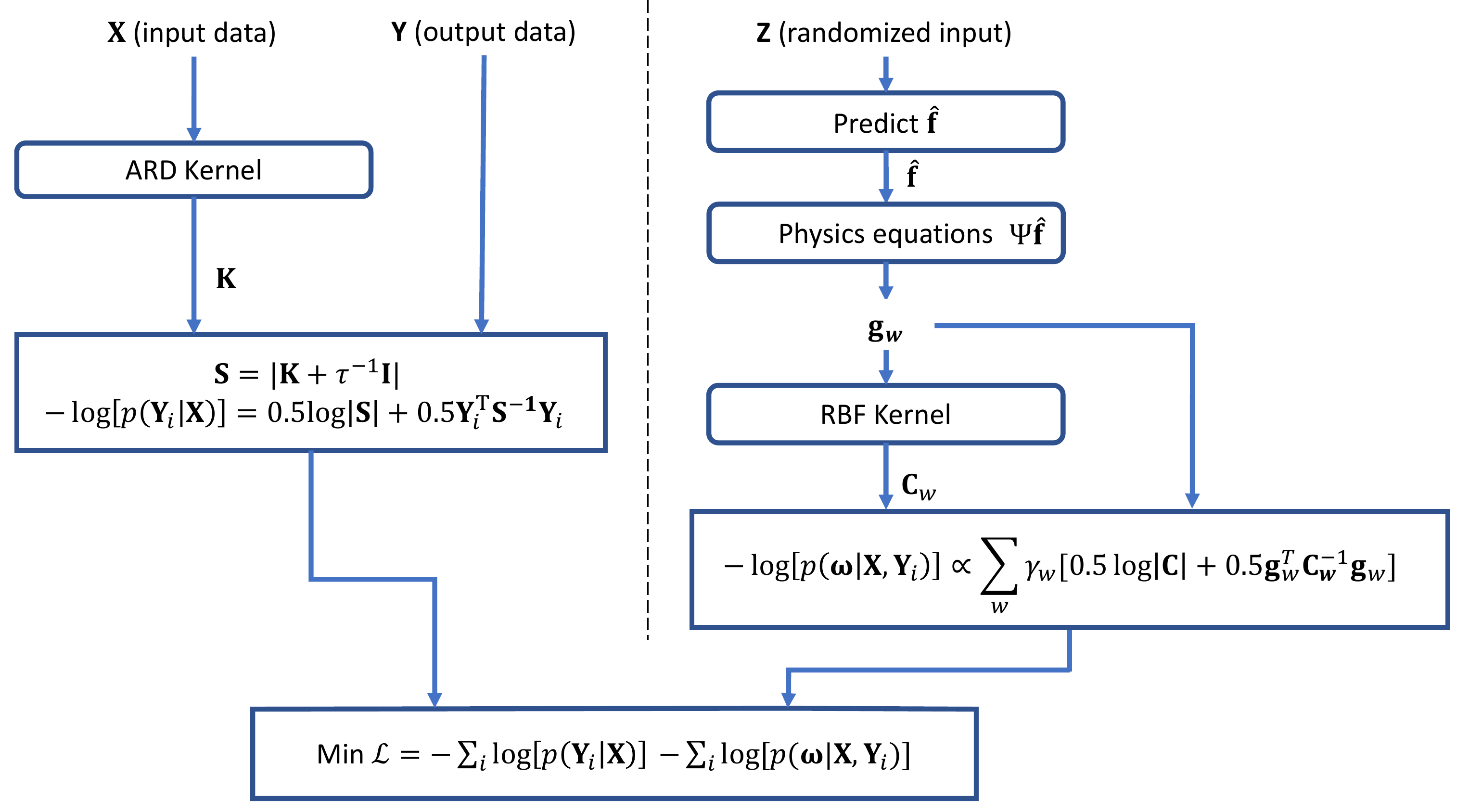}
    \caption{The structure of the proposed loss function}
    \label{fig:implement}
\end{figure} 
\begin{figure}[h!]
    \centering
    \includegraphics[width=0.7\textwidth]{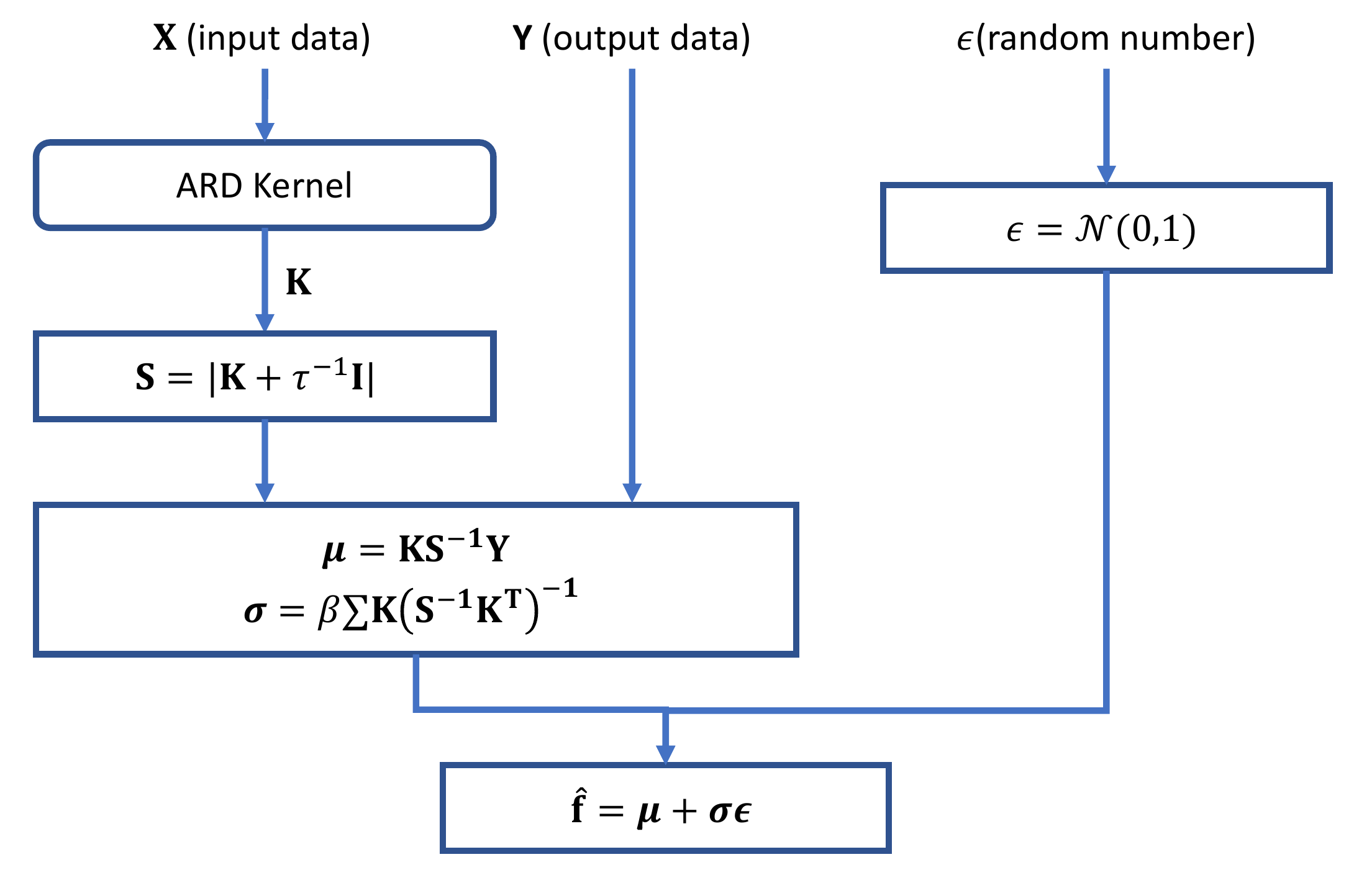}
    \caption{The structure of estimation}
    \label{fig:predict}
       \vspace{-0.2in}  
\end{figure} 

In the experiments, the parameters of the proposed method are set as follows: 
(a) the number of pseudo observations $m=10$, 
(b) the strength of regularization $\lambda$ is fine-tuned numerically.
The proposed inference algorithm is implemented in the Tensorflow framework, where the optimizer ADAM \citep{kingma2014adam} is chosen for updating the parameters.

\subsection{Results Analysis}

\subsubsection{Comparison with Pure Machine Learning Models}
To prove the superiority of the proposed PRML framework compared with pure ML models, this subsection aims to compare the three PRGP models, LWR-PRGP, PW-PRGP, and ARZ-PRGP, with pure GP and other popular ML models such as multilayer perceptron, support vector machine, and random forest \citep{bishop2006pattern}. Also recall that one main contribution of PRML is that it is more explainable in terms of model performance. Hence, this study further adopts another physical model, the well-known heat equation, to prove the indispensability of classical traffic flow models in the PRML framework, since the heat equation is not suitable to model traffic flows. The heat equation is formulated in Eq.~\ref{eq:heat}.
\begin{equation}
    \label{eq:heat}
    \frac{\partial f_h(x,t)}{\partial t}=\beta_1\nabla^2 f_h(x,t)
\end{equation}

Note that the inputs of the proposed PRGP-based methods and classical traffic flow models are different. The latter method often requires the on-ramp and off-ramp flow observations as inputs, while the proposed method assumes unobserved on-ramp and off-ramp flows in the framework and does not require such data.
The training process of each model, with $500$ iterations and $2,880$ samples, costs $10,480$ seconds in average on a workstation equipped with a 3.5GHz 6-core CPU.
In the testing phase, the time complexity of the model estimation is marginal (less than $1$ second) empirically, similar to all ML models.
Note that the computational process can be accelerated by about $5$ time if a NVIDIA CUDA-capable GPU is used.
Figs~\ref{fig:prgp_flow}-\ref{fig:prgp_speed} compare the flow and speed estimations with the ground truth in the studied case. If the coefficient of the trend line is close to $1$ and the intercept is close to $0$, the estimation will be considered as accurate. The results show that both pure GP and proposed PRGP models can perform well in estimating the flows and speeds.
\begin{figure}[h!]
    \centering
        \begin{subfigure}[b]{0.45\textwidth}
        \centering
        \includegraphics[width=\textwidth]{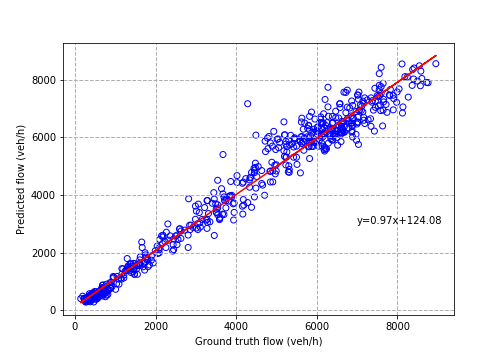}
        \caption{GP}
        \label{fig:gp_flow}
    \end{subfigure}
    \hfill
    \begin{subfigure}[b]{0.45\textwidth}
        \includegraphics[width=\textwidth]{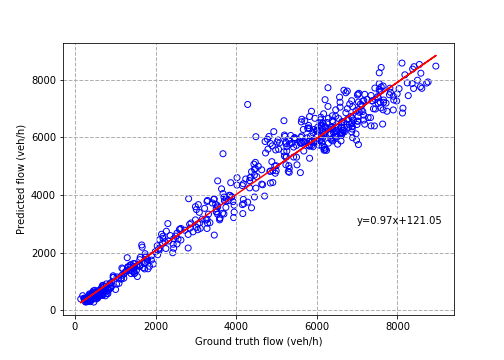}
        \caption{LWR-PRGP}
        \label{fig:pigp_lwr_flow}
    \end{subfigure}
    \\
    \begin{subfigure}[b]{0.45\textwidth}
        \includegraphics[width=\textwidth]{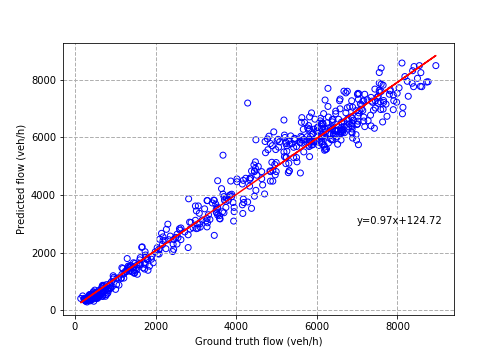}
        \caption{PW-PRGP}
        \label{fig:pigp_pw_flow}
    \end{subfigure}
    \hfill
    \begin{subfigure}[b]{0.45\textwidth}
        \includegraphics[width=\textwidth]{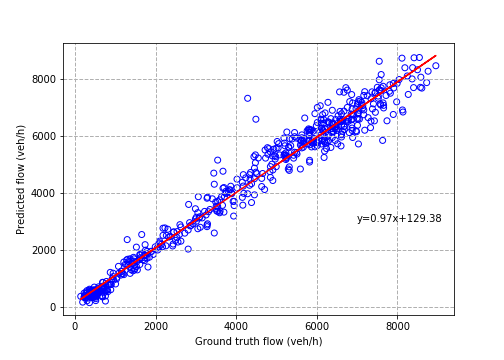}
        \caption{ARZ-PRGP}
        \label{fig:pigp_arz_noise_flow}
    \end{subfigure}
    \caption{Comparison between flow estimation by GP and PRGPs and the ground truth}
    \label{fig:prgp_flow}
       \vspace{-0.2in}
\end{figure}
\begin{figure}[h!]
    \centering
        \begin{subfigure}[b]{0.45\textwidth}
        \centering
        \includegraphics[width=\textwidth]{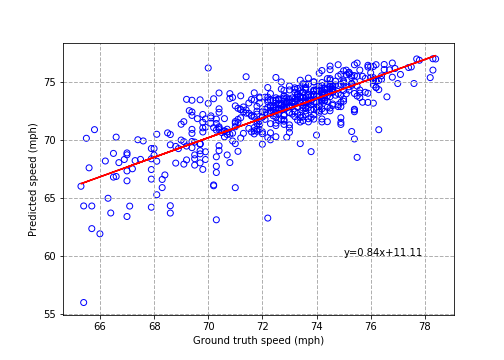}
        \caption{GP}
        \label{fig:gp_speed}
    \end{subfigure}
    \hfill
    \begin{subfigure}[b]{0.45\textwidth}
        \includegraphics[width=\textwidth]{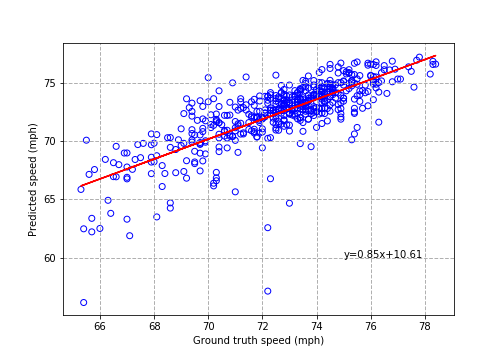}
        \caption{LWR-PRGP}
        \label{fig:pigp_lwr_speed}
    \end{subfigure}
    \\
    \begin{subfigure}[b]{0.45\textwidth}
        \includegraphics[width=\textwidth]{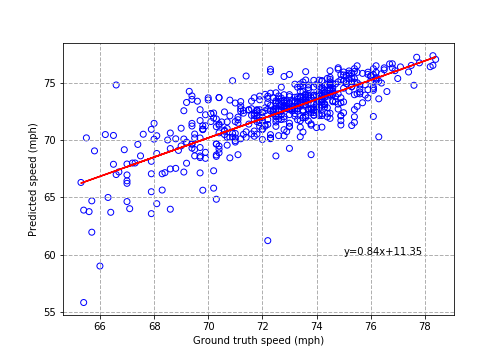}
        \caption{PW-PRGP}
        \label{fig:pigp_pw_speed}
    \end{subfigure}
    \hfill
    \begin{subfigure}[b]{0.45\textwidth}
        \includegraphics[width=\textwidth]{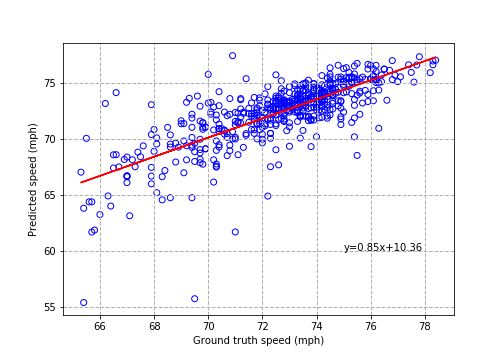}
        \caption{ARZ-PRGP}
        \label{fig:pigp_arz_speed}
    \end{subfigure}
    \caption{Comparison between speed estimations by GP and PRGPs and the ground truth}
    \label{fig:prgp_speed}
       \vspace{-0.2in}
\end{figure}

To quantify the precision of outputs, Rooted Mean Squared Error (RMSE) and Mean Absolute Percentage Error (MAPE) of each dimension are used as the performance metric, which are defined in Eqs.~\ref{eq:def_rmse}-\ref{eq:def_mape}. 
\begin{equation}
RMSE_j = \sqrt{\frac{1}{N}\sum_{i=1}^{N}{\Big(\frac{[\mathbf{y}_j]_{i}-[\hat{\mathbf{f}}_j]_i}{\sigma_i}\Big)^2}}, \forall j\in {1,\ldots,d^\prime}
\label{eq:def_rmse}
\end{equation}
\begin{equation}
MAPE_j = \frac{100\%}{N}\sum_{i=1}^{N}{\Big\vert\frac{[\mathbf{y}_j]_{i}-[\hat{\mathbf{f}}_j]_i}{[\mathbf{y}_j]_{i}}\Big\vert}, \forall j\in {1,\ldots,d^\prime}
\label{eq:def_mape}
\end{equation}


Table~\ref{tab:benchmark1} summarizes of the results of the comparable baselines and the proposed method in the same dataset. Among the four pure ML models, the GP can obviously outperform the other ML models in terms of providing more accurate estimations of both flows and speeds. The GP can yield a 39.74~veh/5-min of RMSE and a 13.70\% of MAPE for flow and a 2.7~mph of RMSE and a 2.64\% for MAPE for speed, while the other three produced much higher RMSEs and MAPEs of both flow and speed estimates. Further comparison between the pure GP and the three PRGP models reveal that PRGP models can improve the accuracy of both flow and speed estimations. However, the improvement is not significant, which is because pure GP can already achieve a very good estimation performance and leaves limited space for improvement by the PRGP. Moreover, to validate PRGP's contribution in making the results more explainable, the comparison with Heat-PRGP, which uses the physical knowledge from the heat equation, shows that a physical model that cannot precisely describe the traffic flow patterns could even downgrade the capability of the PRGP. Another side evidence is that PW and ARZ, which are the improved version of LWR, can improve the performance of the PRGP compared with the LWR.        
\begin{table}[h!]
    \caption{Comparison of the results of the proposed method and the baseline methods}
    \centering
    \begin{tabular}{c|p{2.0cm}|c|p{2.1cm}|c}
    \toprule
         Method & Flow RMSE (veh/5min) & Flow MAPE & Speed RMSE (mph) & Speed MAPE \\
    \midrule
        Multilayer perceptron  & 113.95    &30.80\%& 13.61 &19.91\%\\
        Support Vector Machine  & 124.84    &34.24\%& 9.58  &13.01\%\\
        Random Forest           & 108.24    &27.60\%& 8.66  &12.02\%\\
        pure GP                 & 39.74     &13.70\%& 2.76  & 2.64\%\\
        LWR-PRGP      & 37.19     &12.77\%& 2.96  & 2.65\%\\
        PW-PRGP       & 35.45     &12.42\%& 3.02  & 2.68\%\\
        ARZ-PRGP      & 34.75     &11.48\%& 2.90  & 2.72\%\\
        Heat-PRGP     & 79.51     &23.49\%& 5.20  & 6.75\%\\ 
    \bottomrule
    \end{tabular}
    \label{tab:benchmark1}
\end{table}

\subsubsection{Comparison with physical models (Traffic Flow Models)}
To provide physical baselines for the performance comparison, the LWR, PW, ARZ models are calibrated with the obtained field data. For model calibration, we follow the method by \cite{akwir2018neural}, where the hybrid scheme of neural network and nonlinear partial differential equation is used to dynamically adjust all outputs of the three models to obtain their calibrated parameters.
Figs~\ref{fig:pk_flow}-\ref{fig:pk_speed} plot the estimated flow and speed from the three physical models versus the ground truth. Obviously, the estimation results are quite biased for both flow and speed.
\begin{figure}[h!]
    \centering
        \begin{subfigure}[b]{0.3\textwidth}
        \centering
        \includegraphics[width=\textwidth]{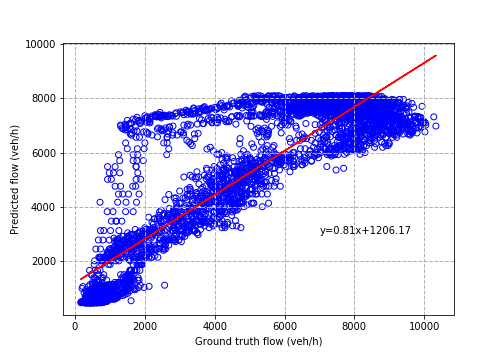}
        \caption{LWR}
        \label{fig:lwr_flow}
    \end{subfigure}
    \hfill
    \begin{subfigure}[b]{0.3\textwidth}
        \includegraphics[width=\textwidth]{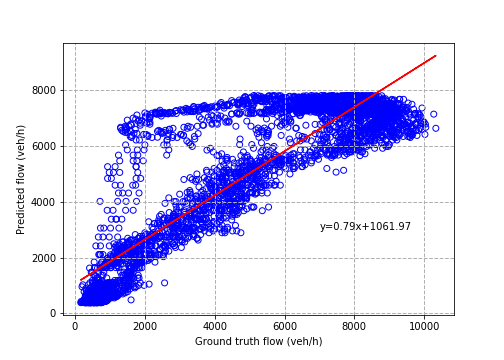}
        \caption{PW}
        \label{fig:pw_flow}
    \end{subfigure}
    \hfill
    \begin{subfigure}[b]{0.3\textwidth}
        \includegraphics[width=\textwidth]{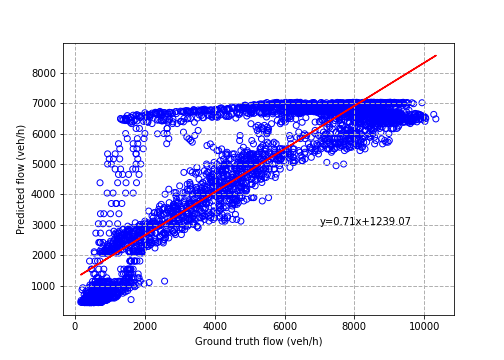}
        \caption{ARZ}
        \label{fig:arz_flow}
    \end{subfigure}
    \caption{Estimated flow by the calibrated physical models v.s. ground truth}
    \label{fig:pk_flow}
\end{figure}
\begin{figure}[h!]
    \centering
        \begin{subfigure}[b]{0.3\textwidth}
        \centering
        \includegraphics[width=\textwidth]{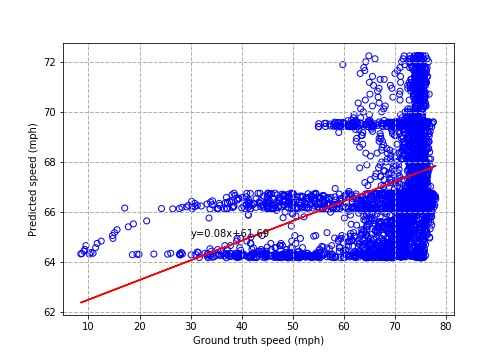}
        \caption{LWR}
        \label{fig:lwr_speed}
    \end{subfigure}
    \hfill
    \begin{subfigure}[b]{0.3\textwidth}
        \includegraphics[width=\textwidth]{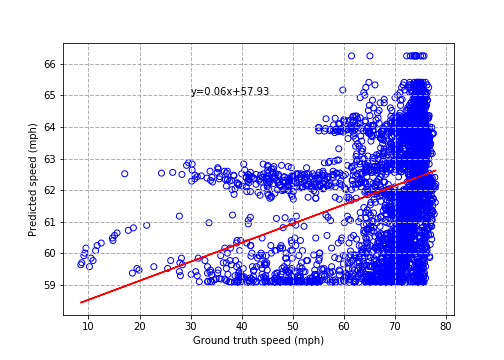}
        \caption{PW}
        \label{fig:pw_speed}
    \end{subfigure}
    \hfill
    \begin{subfigure}[b]{0.3\textwidth}
        \includegraphics[width=\textwidth]{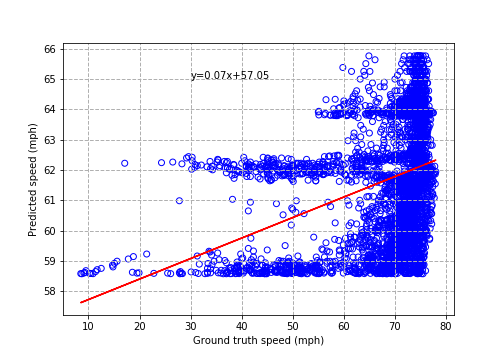}
        \caption{ARZ}
        \label{fig:arz_speed}
    \end{subfigure}
    \caption{Estimated speed by the calibrated physical models v.s. ground truth}
    \label{fig:pk_speed}
\end{figure}

To better justify models' estimation accuracy, Table~\ref{tab:benchmark2} shows the results of proposed method and the calibrated physical models in estimation errors. It can be found that the proposed method significantly outperforms the baseline methods by around $80~veh/5min$ in flow RMSE and $18\%$ in MAPE and $7~mph$ in speed RMSE and $15\%$ in MAPE. Hence, it can be concluded that the estimation performance of traffic flows models can be greatly improved if they are encoded into a ML framework. The real-world uncertainties of flow and speed can be captured by the ML portion properly. 
\begin{table}[h!]
    \caption{Comparison of the results of the proposed methods and the physics-based methods}
    \centering
    \begin{tabular}{c|p{2.0cm}|c|p{2.1cm}|c}
    \toprule
         Method & Flow RMSE (veh/5min) & Flow MAPE & Speed RMSE (mph) & Speed MAPE\\
    \midrule
        Calibrated LWR          & 115.75    &  32.96\%& 9.88  &   14.4\%\\
        LWR-regularized GP      & 37.19     &  12.77\%& 2.96  &   2.76\%\\
        Calibrated PW           & 115.80    &  30.00\%& 10.41 &   18.2\%\\
        PW-regularized GP       & 35.5      &  12.42\%& 3.02  &   2.68\%\\
        Calibrated ARZ          & 155.20    &  32.00\%& 12.71 &   18.4\%\\
        ARZ-regularized GP      & 34.75     &  11.48\%& 2.90  &   2.72\%\\
    \bottomrule
    \end{tabular}
    \label{tab:benchmark2}
\end{table}
\subsubsection{Robustness study}
As aforementioned, the proposed PRML framework is expected to be more robust than pure ML models on noisy dataset. Hence, in this subsection, $50\%$ of the training data is replaced by the flawed data, which are generated with $100~veh/5min$ noises in flows, and the testing data keep unchanged.
Notably, for model evaluations, the testing dataset is not mixed with noises. Also, since GP can outperform multilayer perceptron, support vector machine, and random forest in both flow and speed estimation, we will only examine the robustness of GP and PRGP in this subsection.
Table~\ref{tab:noise} and Figs~\ref{fig:noise_flow}-\ref{fig:noise_speed} summarize their estimation performance on the noised training data.
The results show that the GP has limited resistance to high biased data, e.g., caused by traffic detector malfunctions.
The three PRGP models can greatly outperform pure GP by about 160~veh/h of RMSE and over 100\% of MAPE in flow estimations. Hence, it can be concluded that the proposed PRML framework are much more robust than the pure ML models when the input data is subject to unobserved random noise. This is due to PRML's capability of adopting physical knowledge to regularized the ML training process. 
The results also show that heat equation does not capture the dynamics of the traffic flow, and only the well-developed traffic flow model can improve the accuracy of Gaussian process. 
\begin{table}[h!]
    \caption{Comparison of the estimation accuracy with noisy training dataset}
    \centering
    \begin{tabular}{c|p{2.0cm}|c|p{2.1cm}|c}
    \toprule
        Method & Flow RMSE (veh/5min) & Flow MAPE & Speed RMSE (mph) & Speed MAPE\\ 
    \midrule
        pure GP &   212.17  & 135.19\%&   5.96    & 3.35\% \\
        GP-LWR  &   41.78   &   9.73\%&   6.01    & 3.46\% \\
        GP-PW   &   41.11   &   9.60\%&   4.43    & 3.30\%\\
        GP-ARZ  &   35.37   &   9.51\%&   3.06    & 2.72\%\\
        GP-HEAT &   215.01  & 138.29\%&   4.31    & 33.6\%\\
    \bottomrule
    \end{tabular}
    \label{tab:noise}
\end{table}
\begin{figure}[h!]
    \centering
        \begin{subfigure}[b]{0.45\textwidth}
        \centering
        \includegraphics[width=\textwidth]{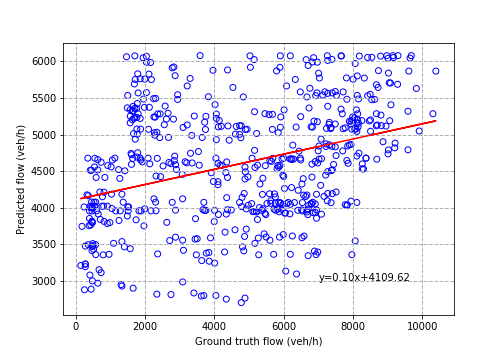}
        \caption{GP}
        \label{fig:gp_noise_flow}
    \end{subfigure}
    \hfill
    \begin{subfigure}[b]{0.45\textwidth}
        \includegraphics[width=\textwidth]{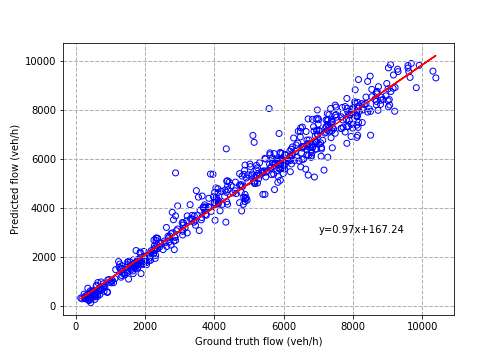}
        \caption{GP-LWR}
        \label{fig:pigp_lwr_noise_flow}
    \end{subfigure} 
    \\
    \begin{subfigure}[b]{0.45\textwidth}
        \includegraphics[width=\textwidth]{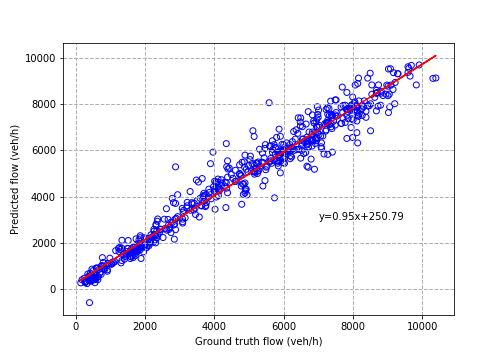}
        \caption{GP-PW}
        \label{fig:pigp_pw_noise_flow}
    \end{subfigure}
    \hfill
    \begin{subfigure}[b]{0.45\textwidth}
        \includegraphics[width=\textwidth]{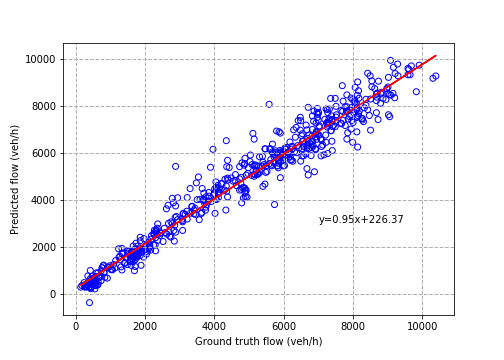}
        \caption{GP-ARZ}
        \label{fig:pigp_arz_noise_flow}
    \end{subfigure}
       \vspace{-0.1in}
    \caption{Comparison between flow estimation and ground truth with noisy training dataset}
    \label{fig:noise_flow}
       \vspace{-0.2in}
\end{figure}
\begin{figure}[h!]
    \centering
    \begin{subfigure}[b]{0.45\textwidth}
        \centering
        \includegraphics[width=\textwidth]{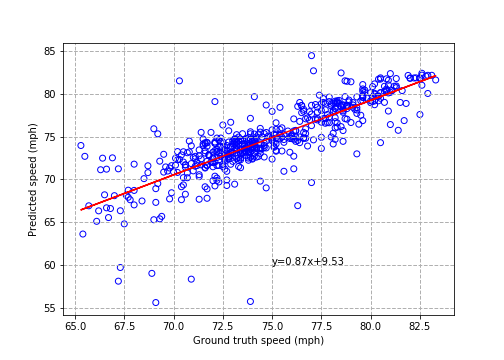}
        \caption{GP}
        \label{fig:gp_noise_speed}
    \end{subfigure}
    \hfill
    \begin{subfigure}[b]{0.45\textwidth}
        \centering
        \includegraphics[width=\textwidth]{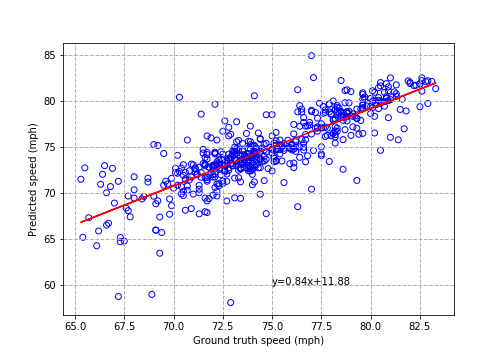}
        \caption{GP-LWR}
        \label{fig:pigp_lwr_noise_speed}
    \end{subfigure}
    \\
    \begin{subfigure}[b]{0.45\textwidth}
        \centering
        \includegraphics[width=\textwidth]{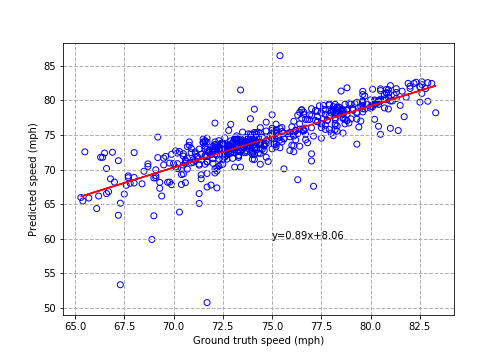}
        \caption{GP-PW}
        \label{fig:pigp_pw_noise_speed}
    \end{subfigure}
    \hfill
    \begin{subfigure}[b]{0.45\textwidth}
        \includegraphics[width=\textwidth]{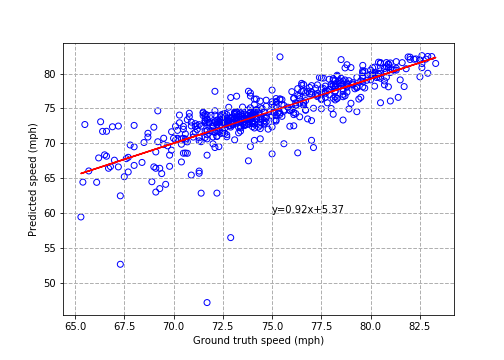}
        \caption{GP-ARZ}
        \label{fig:pigp_arz_noise_speed}
    \end{subfigure}
    \vspace{-0.1in}
    \caption{Comparison between speed estimation and ground truth with noisy training dataset}
    \label{fig:noise_speed}
       \vspace{-0.2in}
\end{figure}

\section{Conclusions and Future Research Directions}\label{sec:5}
In the literature, traffic flow models have been well developed to explain the traffic phenomena, however, have theoretical difficulties in stochastic formulations and rigorous estimation.
In view of the increasing availability of data, the data-driven methods are prevailing and fast-developing, however, have limitations of lacking sensitivity of irregular events and compromised effectiveness in sparse data.
To address the issues of both methods, a hybrid framework to incorporate the advantages of both methods is investigated.
This paper proposes a stochastic modeling framework to capture the random detection noise and the latent unobserved of traffic data as well as leveraging the well-defined fundamental diagram, conservation law and momentum conditions.
The traffic state indicators (i.e. flow, speed, density) are assumed to be multi-variant Gaussian distributed.
A physics regularized Gaussian process (PRGP) is proposed to encode the physics knowledge in the Bayesian inference structure as the shadow Gaussian process.
The shadow Gaussian process is proven to regularize the conventional constraint-free Gaussian process as a soft constraint.
To estimate the proposed PRGP, a posterior regularized inference algorithm is derived and implemented with auto-differentiation libraries.
The computational complexity is cubic of the product of the sample size and the output dimension $O((Nd^\prime)^3+m^3)$.
A preliminary real-world case study is conducted on PeMS detection data collected from a freeway segment in Utah and the well-known continuous traffic flow models (i.e. LWR, PW, ARZ) are tested.
In comparison to the pure machine learning methods and pure physical models, the numerical results justify the effectiveness and the robustness of the proposed method.

The potential directions for future Research may include: 
(1) extending the proposed method to leverage other models for traffic state estimation, such as discrete macroscopic traffic flow model regularized Gaussian process; 
(2) extending the proposed method to solve other problems, such as microscopic behavior models regularized Gaussian process for vehicle trajectory prediction; 
(3) extending the physics regularization methodology in other machine learning algorithms, such as random forest and support vector machine to combine general physics knowledge in learning tasks.

\bibliography{main}

\end{document}